\documentclass[11pt]{article}

\usepackage{palatino}
\usepackage[margin = 1in]{geometry}
\usepackage{microtype}
\usepackage{graphicx}
\usepackage{subfigure}
\usepackage{booktabs} 

\usepackage{hyperref}

\usepackage{amsmath}
\usepackage{amssymb}
\usepackage{mathtools}
\usepackage{amsthm}

\usepackage[capitalize,noabbrev]{cleveref}

\usepackage[textsize=tiny]{todonotes}

\usepackage[round]{natbib}

\usepackage{bbm}
\usepackage{thmtools,thm-restate}

\usepackage{mathrsfs}
\usepackage{amsmath}
\usepackage[english]{babel}
\usepackage[toc,page]{appendix} 
\usepackage{hyperref}
\usepackage{graphicx,wrapfig,lipsum,xcolor}	
\usepackage{yfonts}	 		
\usepackage{amssymb}
\usepackage{amsmath}
\usepackage{amsthm}
\usepackage{scalerel}
\usepackage{verbatim}
\usepackage[lined,boxed,ruled,norelsize,algo2e,linesnumbered]{algorithm2e}
\usepackage{dsfont}
\usepackage{amsmath}
\usepackage{mathrsfs}
\usepackage{amssymb, graphicx, amsmath, amsthm}

\def \R{\mathbb{R}}

\def \calH{\mathcal{H}}

\def \calP{\mathcal{P}}
\def \calM{\mathcal{M}}
\def \calZ{\mathcal{Z}}

\def \E{\mathbb{E}}

\DeclareMathOperator*{\tr}{Tr}
\DeclareMathOperator*{\id}{id}

\DeclareMathOperator*{\dist}{dist}
\DeclareMathOperator*{\Lip}{Lip}  
\DeclareMathOperator*{\vol}{vol}

\DeclareMathOperator*{\argmin}{arg\,min}

\usepackage{algorithm}
\usepackage{algorithmic}
\usepackage{tikz}

\newtheorem{definition}{Definition}
\newtheorem{example}{Example}
\newtheorem{theorem}{Theorem}

\newtheorem{lemma}{Lemma}
\newtheorem{corollary}{Corollary}
\newtheorem{remark}{Remark}

\usepackage{amsthm}

\allowdisplaybreaks

\usepackage{booktabs}
\usepackage{threeparttable}

\definecolor{darkblue}{rgb}{0.0,0.0,0.65}
\definecolor{darkred}{rgb}{0.68,0.05,0.0}
\definecolor{darkgreen}{rgb}{0.0,0.29,0.29}
\definecolor{darkpurple}{rgb}{0.47,0.09,0.29}
\usepackage{hyperref}
\hypersetup{
   colorlinks = true,
   citecolor  = darkblue,
   linkcolor  = darkred,
   filecolor  = darkblue,
   urlcolor   = darkblue,
 }

\title{{Sample Complexity Bounds for Estimating Probability Divergences under Invariances}}

\date{}

\author{Behrooz Tahmasebi\\MIT CSAIL\\\texttt{bzt@mit.edu} \and Stefanie Jegelka\\TU Munich and MIT CSAIL\\\texttt{stefje@mit.edu}}

\begin{document}

\maketitle 
 
\begin{abstract}
Group-invariant probability distributions appear in many data-generative models in machine learning, such as graphs, point clouds, and images. In practice, one often needs to estimate divergences between such distributions. In this work, we study how the inherent invariances, with respect to any smooth action of a Lie group on a manifold, improve sample complexity when estimating the 1-Wasserstein distance, the Sobolev Integral Probability Metrics (Sobolev IPMs), the Maximum Mean Discrepancy (MMD), and also the complexity of the density estimation problem (in the $L^2$ and $L^\infty$ distance). Our results indicate a two-fold gain: (1) reducing the sample complexity by a multiplicative factor corresponding to the group size (for finite groups) or the normalized volume of the quotient space (for groups of positive dimension); (2) improving the exponent in the convergence rate (for groups of positive dimension). These results are completely new for groups of positive dimension and extend recent bounds for finite group actions.
\end{abstract}

\section{Introduction}

Estimating the optimal transportation distance \citep{villani2021topics, villani2009optimal,  santambrogio2015optimal}  between probability measures is a fundamental problem in statistics, with many applications in machine learning, from Generative Adversarial Networks (GANs) \citep{goodfellow2020generative, arjovsky2017wasserstein, salimans2018improving, mallasto2019q} 
to domain adaptation and generalization \citep{flamary2016optimal, courty2014domain, chuang2021measuring}, geometric data processing (e.g., Wasserstein barycenters \citep{cuturi2014fast} and intrinsic dimension estimation \citep{block2022intrinsic}), biomedical research \citep{zhang2021review}, and control and dynamical systems \citep{bunneoptimal}.

Estimating the 1-Wasserstein distance is known to be a difficult task in general, and it suffers from the curse of dimensionality \citep{tsybakov09nonparam}. The slow convergence rate is generally unimprovable, as there exist  probability measures that are difficult to estimate.  
However, in many applications (e.g.,  graphs, point clouds, molecules, spectral data), the underlying probability measures are \textit{invariant}
 with respect to a group action on the input space. 
As observed in recent works \citep{birrell2022structure, chen2023sample}, considering the group invariances in the mathematical model can help improve the convergence rate of the 1-Wasserstein distance and the Sobolev Integral Probability Metrics (Sobolev IPMs), with applications, e.g.,  to generative models for invariant data.

\begin{table*}[t]
  \caption{Summary of the main results.}
  \vspace{0.1in}
  \label{sample-table}
  \centering
  \begin{tabular}{lll}
    \toprule
    Divergence Measure     & $\delta_G$     & $\kappa_G$  \\
    \midrule
    1-Wasserstein Distance (Theorem \ref{thrm}) & $\vol(\calM/G)$  & $1/d$   \\
    1-Wasserstein Distance ($s$-smooth, Theorem \ref{thrmsob})    & $\vol(\calM/G)$  & $(s+1)/(2s+d)$ \\
   Sobolev IPMs ($\alpha <d/2$, Theorem \ref{thrmsobipm})    &  $\vol(\calM/G)$ &  $(s+\alpha)/(2s+d)$ \\
   Sobolev IPMs ($\alpha>d/2$, MMD regime,  Theorem \ref{thrmmmd})    & $\mathcal{Z}(\alpha;G)$       & 1/2 \\
   Density Estimation in $L^2$ Distance (Theorem \ref{thrmdenl})    & $\vol(\calM/G)$        & $s/(2s+d)$ \\
   Density Estimation in $L^\infty$ Distance ($s>d/2$, Theorem \ref{thrmdenlinf})   & $(\vol(\calM/G))^{\frac{2s}{s-d/2}}$        & $(s-d/2)/(2s+d)$ \\
    \bottomrule 
  \end{tabular}
   \begin{tablenotes}
       \item *For any divergence measure $D$ in the table, we prove that there exists an estimator $\tilde{\mu}$, as a function of $n$ i.i.d. samples from $\mu$, such that $\E[D(\mu,\tilde{\mu})] \lesssim  \Big(\frac{\delta_G}{n}\Big)^{\kappa_G}$. Sobolev IPMs use the Sobolev space $\calH^\alpha(\calM)$ as test functions, for some $\alpha \ge 0$. Moreover, $d\mu/dx \in \calH^s(\calM)$ for some $s\ge 0$. Also, $d$ denotes the dimension of the quotient space $\calM/G$, and  $\mathcal{Z}(\alpha; G)$ is the zeta function of the manifold when it is restricted to the invariant eigenfunction of the Laplace-Beltrami operator with respect to the group $G$.
    \end{tablenotes}
\end{table*}

In this paper, we study the sample complexity gain of invariances for estimating probability measures in the 1-Wasserstein distance, the Sobolev IPMs, the Maximum Mean Discrepancy (MMD), and also for the density estimation problem (in $L^2$ and $L^\infty$ distance). We consider  probability measures supported on a given\footnote{ Throughout this paper, we always assume that the underlying space is a known manifold; this is distinguished from a related body of work on estimating distributions supported on unknown low-dimensional embedded manifolds \citep{divol2022measure}.} connected compact smooth manifold $\calM$ that are invariant with respect to a smooth action of a Lie group $G$ on $\calM$. 
In this general setting, given a  (Borel) probability measure $\mu$, supported on the manifold $\calM$, we prove that there exists an estimator $\tilde{\mu}$ (as a function of $n$ i.i.d. samples from $\mu$), such that 
\begin{align}
\E[D(\mu,\tilde{\mu})] \lesssim  \Big(\frac{\delta_G}{n}\Big)^{\kappa_G},
\end{align}
where $D$ can be replaced by the 1-Wasserstein distance,  Sobolev IPMs,  MMD, and also the $L^2(\calM)$ and $L^\infty(\calM)$ distance between distributions (the density estimation problem in $L^2(\calM)$ and $L^\infty(\calM)$). 
The exponent $\kappa_G$ and the factor $\delta_G$ depend on the distribution's smoothness properties and the quotient space's dimension and volume.

The new sample complexity bound shows two different aspects of gain of invariances. First, the quantity $\kappa_G$ is observed to be a non-increasing function of the quotient space's dimension $d$, so compared to the general case (i.e., without invariances, $G = \{ \text{id}_G\}$), the exponent is improved,  as $d$ can be potentially as small as $\dim(\calM) - \dim(G)$.

Second, the factor $\delta_G$ is a non-decreasing function of the quotient space's volume  $\vol(\calM/G)$, which can be potentially much smaller than $\vol(\calM)$. For instance, it can be $\vol(\calM)/|G|$ for finite groups. In particular, we have $\delta_G = \vol(\calM/G)$ for estimating the 1-Wasserstein distance. Therefore,  for finite groups (i.e., $\dim(G) = 0$), we may view the gain of invariances for sample complexity (compared to the general case) as replacing the number of samples $n$ by $n\times |G|$ in the classical convergence rate of the 1-Wasserstein distance estimation (without invariances). This is intuitively reasonable: 
the gain of invariances for finite groups is that each sample effectively conveys the information of $|G|$ samples when comparing the invariant case to the general case. Our result proves this intuition formally.

The upper bound proved in this paper is completely new for groups of positive dimension. For finite groups, it extends a recent result for submanifolds (of full dimension) of $\R^d$ under $1$-Lipschitz group actions \citep{chen2023sample}  to arbitrary manifolds and arbitrary Lie groups.

We further study the convergence rate of estimating probability divergences for smooth distributions. Indeed, for probability measures having a density with respect to the uniform distribution on the manifold, with square-integrable derivatives up to order $s$ (known as being in the Sobolev space $\calH^s(\calM)$), we prove upper bounds on the convergence rate  that  exhibit the same two-fold gain for the sample complexity as discussed above, namely a factor in the \textit{effective} number of samples, and an improvement in the exponent.  Note that all the proven upper bounds in the paper reduce to the known tight bounds on estimating without invariances if we set $G = \{ \text{id}_G\}$ (i.e., the trivial group). 

Our findings cannot be derived immediately from known results on estimating the 1-Wasserstein distance or   Sobolev IPMs. Instead of the idea of using covering numbers, which are used in a recent work \citep{chen2023sample}, we use a Fourier approach to bounding the error. Specifically, we use the theory of the Laplace-Beltrami operator on manifolds, and via a new version of Weyl's law, which captures the sparsity of the Fourier series on manifolds, as well as ideas from differential geometry and Fourier analysis, we prove the main result.

In short, in this paper, we make the following contributions:
\begin{itemize}
\item We prove convergence rates for estimating the 1-Wasserstein distance for any invariant probability measure  to a smooth Lie group action on a connected compact manifold (Theorem \ref{thrm}). Moreover, we extend this result to smooth distributions under invariances (Theorem \ref{thrmsob}).

\item We prove convergence rates for estimating the Sobolev IPMs (Theorem \ref{thrmsobipm}) and the Maximum Mean Discrepancy (MMD, Theorem \ref{thrmmmd}) under group invariances.

\item We prove convergence rates for the density estimation problem in the  $L^2(\calM)$ and $L^\infty(\calM)$ distance under invariances (Theorem \ref{thrmdenl}  and Theorem \ref{thrmdenlinf}).  
\end{itemize}

\section{Related Work}

Optimal transportation has been an extensive area of research for the last few decades \citep{villani2021topics, villani2009optimal}. The Kantorovich duality \citep{kantorovich2006translocation} allows the dual formulation of the optimal transportation problem used in this paper. 
The study of the convergence rate of the 1-Wasserstein distance is a classical problem   \citep{tsybakov09nonparam, fournier2015rate, weed2019sharp, boissard2014mean, singh2018nonparametric,  
lei2020convergence, 
butkovsky2014subgeometric, 
panaretos2019statistical, 
arjovsky2017wasserstein, mroueh2017sobolev, mena2019statistical, rigollet2022sample, genevay2019sample}.  \citet{canas2012learning} use  the 1-Wasserstein distance for learning probability measures on manifolds. The computational aspects of optimal transport in machine learning are also addressed in \citep{peyre2019computational}. 

Group-invariant probability measures have many applications in machine learning; for example, in addition to what was presented in the previous section, in group equivariant GANs \citep{dey2020group, birrell2022structure}, normalizing flow  \citep{bilovs2021scalable} and equivariant flows  \citep{kohler2019equivariant}.

In a recent paper, \citet{chen2023sample} studied the gain of invariances for estimating the 1-Wasserstein distance, and they proved upper bounds on the convergence rate for finite groups with $1$-Lipschitz actions of submanifolds (of full dimension)
 of $\R^d$, while our rates hold for any group (including groups of positive dimension), and any manifold, thus extending the previous results. In another closely related work, the gain of group invariances is investigated for the kernel ridge regression problem on manifolds \citep{tahmasebi2023exact}. While we study completely different divergence estimation problems here, we follow the approach proposed in \citep{tahmasebi2023exact} to use the Laplace-Beltrami operator (i.e., Fourier basis) to study the benefits of invariances. For more on  the benefits of invariances, see \citep{bietti2021sample}.

It is known that the convergence rate of estimating the 1-Wasserstein distance can be improved if the densities are smooth \citep{singh2018minimax, fournier2015rate, liang2021well, niles2022minimax}. The study of convergence rates of the Sobolev IPMs for smooth distributions has applications, e.g., in Sobolev GANs \citep{mroueh2017sobolev}. See 
\citep{fukumizu2007kernel, gretton2006kernel, gretton2007kernel, smola2007hilbert,  gretton2012kernel, anderson2019cormorant, tolstikhin2017minimax, tolstikhin2016minimax,  sriperumbudur2011universality, sriperumbudur2010hilbert} for the study of achievable convergence rates and other properties of MMDs and \citep{hendriks1990nonparametric} for the density estimation problem (without invariances).

\section{Preliminaries}

Let $\calM$ denote an arbitrary compact, connected, and smooth manifold without boundary\footnote{The results can also be generalized to manifolds with boundaries. But we consider boundaryless manifolds here for simplicity.}. Let $\calP(\calM)$ denote the set of Borel probability measures on $\calM$, and also  let $\Lip(\calM)$ denote the set of all measurable functions $f:\calM \to \R$ such that $|f(x) - f(y)| \le \dist(x,y)$, for all $x,y \in \calM$, where $\dist(.,.)$ denotes the geodesic distance between points on $\calM$.   

\begin{definition}
    The 1-Wasserstein distance between any two $\mu,\nu \in \calP(\calM)$ is defined as follows:
\begin{align}
W_1(\mu,\nu):=\sup_{f \in \Lip(\calM) } \Big \{ \int_{\calM} f d\mu  -  \int_{\calM} f d\nu \Big \}\label{emp_rate}.
\end{align}
\end{definition}

For any (unknown) probability measures $\mu,\nu \in \calP(\calM)$,  assume that we are given independent samples $X_1,X_2,\ldots, X_n \overset{\text{i.i.d.}}{\sim} \mu$ and  $Y_1,Y_2,\ldots, Y_n \overset{\text{i.i.d.}}{\sim} \nu$. The goal is to estimate  $D(\mu,\nu)$, where $D$ denotes a metric distance between distributions (such as the 1-Wasserstein distance or a Sobolev Integral Probability Metric),  using the $2n$ independent samples. 
Note that given estimations of the (unknown) probability measures $\mu,\nu$ as a function of their i.i.d. samples, such as $\tilde{\mu},\tilde{\nu}$, one can estimate  $D(\mu,\nu)$ by the triangle inequality\footnote{For more on tightness of this upper bound, see \citep{liang2019estimating}.}:
\begin{align}
 |D(\tilde{\mu}, \tilde{\nu}) - D(\mu,\nu)| \le D(\mu, \tilde{\mu})+D(\nu, \tilde{\nu}). 
\end{align}
Thus, $D(\tilde{\mu},\tilde{\nu})$ gives an estimation of the true distance      $D(\mu,\nu)$. This means that to study the convergence of the probability divergences $D(\mu,\nu)$,  one just needs to prove an upper bound on $D(\mu, \tilde{\mu})$ for any probability measure $\mu \in \calP(\calM)$. In particular, we can focus on the problem of estimating probability measures from samples in the $D(.,.)$ distance, and study the effects of group invariances on its sample complexity. 

Let $G$ be an arbitrary Lie group acting smoothly on $\calM$. Without loss of generality, we assume that $\calM$ is equipped with a Riemannian metric $g$ such that the action of $G$ is \textit{isometric} on $\calM$ with respect to $g$; see \cite{tahmasebi2023exact} for more details. 
A probability measure $\mu \in \calP(\calM)$ is called $G$-invariant if for any Borel set $A\subseteq \calM$ and all $\tau \in G$, one has $\mu(A) = \mu(\tau A)$. For example, the uniform distribution on $(\calM,g)$ is invariant with respect to any isometric group action. Without loss of generality and just for simplicity,  in this paper, we assume that $\vol(\calM) = 1$, and we denote the volume element on the manifold by $dx:=d\vol_g(x)$.


\section{Main Results}\label{sec_mr}

This section presents our main results on bounding $D(\mu, \tilde{\mu})$ under invariances. We start with the 1-Wasserstein distance for arbitrary probability measures (Theorem \ref{thrm}) and then state the results for smooth distributions (Theorem \ref{thrmsob}). Then, we study the sample complexity of Sobolev IPMs (Theorem \ref{thrmsobipm}) and MMDs (Theorem \ref{thrmmmd}).  Finally, we  state our results for the density estimation problem in the $L^2(\calM)$ distance (Theorem \ref{thrmdenl}) and in the $L^\infty(\calM)$ distance (Theorem \ref{thrmdenlinf}).  

\subsection{1-Wasserstein Distance}

For any probability measure $\mu \in \calP(\calM)$, assume that we are given samples $X_1,X_2,\ldots, X_n \overset{\text{i.i.d.}}{\sim} \mu$. Let $\hat{\mu}: = \frac{1}{n}\sum_{i=1}^n \delta_{X_i}$ denote the empirical measure  given $X_1,X_2,\ldots, X_n$, where $\delta_{x}$ denotes the Dirac measure supported on $x \in \calM$. The empirical measure $\hat{\mu}$ is probably the most straightforward way to generate an estimator of $\mu$ from samples. 
However, for $G$-invariant $\mu$, the empirical measure $\hat{\mu}$ is not necessarily $G$-invariant. 
To address this issue, we introduce and use a modified empirical measure, where the estimation is restricted to only the non-trivial parts of the Fourier transform of the measure (see the proof of Theorem~\ref{thrm}), as for measures that exhibit invariance, parts of the spectrum that correspond to non-invariant functions are zero. Thus, the modified empirical measure is always $G$-invariant. 
The first result of this paper is the following theorem on the convergence of the 1-Wasserstein distance for arbitrary (Borel) probability measures using the modified empirical measure.

\begin{restatable}[Convergence rate of the 1-Wasserstein distance under invariances]{theorem}{thrm}\label{thrm}
For any $G$-invariant probability measure $\mu \in \calP(\calM)$, there exists an estimator $\tilde{\mu} \in \calP(\calM)$,  as a function of  $n$ i.i.d. samples $X_1,X_2\ldots, X_n \sim \mu$, such that
\begin{align}
\E[W_1(\mu,\tilde{\mu})] \lesssim \Big ( 
\frac{\vol(\calM/G)}{n}
\Big)^{\frac{1}{d}},
\end{align}
where $\vol(\calM/G)$ is the volume of the quotient space $\calM/G$, and $d := \dim(\calM/G)\ge 3$. Also, the constant can only depend on the manifold $\calM$.  Consequently, one has
\begin{align}
\E[|W_1(\tilde{\mu}, \tilde{\nu}) - W_1(\mu,\nu)|] \lesssim  \Big ( 
\frac{\vol(\calM/G)}{n}
\Big)^{\frac{1}{d}},
\end{align}
for any $G$-invariant probability measures $\mu,\nu \in \calP(\calM)$. 
\end{restatable}

In proving this theorem, there are some complications to overcome. For instance, even though we used the notation $\dim(\calM/G)$ and $\vol(\calM/G)$, we notice that the quotient space $\calM/G$ is \textit{not} necessarily a manifold. Also, it may exhibit boundary,  even though $\calM$ is assumed to be boundaryless; 
  this can make the Fourier approach inapplicable to this problem \citep{tahmasebi2023exact}. To address this issue, we 
  define $\dim(\calM/G)$ (or $\vol(\calM/G)$) as the dimension (or the volume) of the \textit{principal part} of the quotient space. The principal part, denoted as $\calM_0/G$, is a connected dense subset of $\calM/G$, such that it has a manifold structure inherited from $\calM$. Since it is a manifold, one can define its dimension/volume in a natural way. It is guaranteed that the principal part exists, and is unique, under the assumptions of this paper. Besides the principal part, $\calM/G$ is only a disjoint union of  \textit{finitely} many other manifolds, all of lower dimension than the principal part. 
Note that $\vol(\calM/G)$ is defined with respect to the dimension of the quotient space $\dim(\calM/G)$, so it is nonzero even if $\dim(\calM/G)< \dim(\calM)$.

To compare the convergence rate with the general case (i.e., not necessarily $G$-invariant probability measures), note that if $G = \{ \text{id}_G\}$, then the convergence rate is $\E[W_1(\mu,\hat{\mu})] \lesssim \Big ( 
\frac{\vol(\calM)}{n}
\Big)^{\frac{1}{\dim(\calM)}},$ as expected from the standard results for arbitrary probability  measures \citep{fournier2015rate}. 
This shows that the sample complexity of estimating the 1-Wasserstein distance is improved under invariances; (1) the new exponent is $\frac{1}{d}$ with $d = \dim(\calM /G)$, which can be potentially much greater than $\frac{1}{\dim(\calM)}$, (2) the number of samples $n$ is multiplied by $\vol(\calM)/\vol(\calM/G)$. For finite groups,  if the action on $\calM$ is effective\footnote{The action of a group $G$ on a manifold $\calM$ is called effective if any $ \text{id}_G \neq \tau \in G$ corresponds to a  non-trivial bijection on $\calM$.},  then Theorem \ref{thrm} shows that 
\begin{align}
\E[W_1(\mu,\hat{\mu})] \lesssim \Big ( 
\frac{\vol(\calM)}{n|G|}
\Big)^{\frac{1}{\dim(\calM)}}.
\end{align}
This could be interpreted as each sample being worth the same as $|G|$ samples under invariances compared to  the general (non-invariant) case.  
This improves a  recent result on the convergence of the 1-Wasserstein metric under invariances \citep{chen2023sample}. \citet{chen2023sample} prove that this rate is achievable for finite group actions on a compact submanifold (of full dimension) of the Euclidean space $\R^d$.  However, our result is more general, holding for arbitrary smooth compact manifolds, including spheres, tori, and hyperbolic spaces, and also for arbitrary groups, not necessarily finite groups. Indeed, to the best of our knowledge,  the improvement in the exponent is new for the convergence of the 1-Wasserstein distance under invariances.

Let us observe the result of Theorem \ref{thrm} in the following example.

\begin{example}
Consider a point cloud as a set $\{ p_1,p_2,\ldots,p_m\} \subseteq (\R / \mathbb{Z})^3$ of $m$ points on 3-torus. For fixed $m$, we can think of each point cloud as a point on the manifold $ (\R / \mathbb{Z})^{3m}$.  Point clouds are assumed to be unchanged under a change of coordinates for all the points:
\begin{align}
\{ p_1,p_2,\ldots,p_m\} \cong \{A p_1,Ap_2,\ldots,Ap_m\},
\end{align}
for any orthogonal matrix $A$. Also, permuting the points will not change the point clouds.
Let $G$ denote the group of invariances for point clouds defined on   $ (\R / \mathbb{Z})^{3m}$ as above.  Then, by Theorem~\ref{thrm},  the gain of invariances (i.e., estimating the 1-Wasserstein distance on point clouds by considering the invariances of the problem) is (1) improving the exponent from $3m$ to $3m-6$, and (2) multiplying the number of samples $n$ by  $m!$.

\end{example}

\textbf{Proof sketch for Theorem} \ref{thrm}.   In this part, we give a quick proof sketch for Theorem \ref{thrm}. The complete proof is available in Appendix \ref{appendix_thrm}.

To prove the theorem, we focus on an approach for upper bounding the 1-Wasserstein distance using the orthonormal basis $\phi_{\ell} \in L^2(\calM)$, $\ell =0,1,\ldots$, of eigenfunctions of the Laplace-Beltrami operator on $\calM$ in $L^2(\calM)$ (see Appendix \ref{appnd_prel} for more details).
This allows us to conclude that
\begin{align}
W_1^2(\mu,\nu) \le &\sum_{\ell=1}^{\infty} \frac{(\mu_{\ell} - \nu_{\ell})^2 }{\lambda_{\ell}},
\end{align}
where $\mu_{\ell} = \int_{\calM} \phi_{\ell} d\mu$ for each $\ell$ (defined similarly for $\nu$), and $\lambda_{\ell}$, $\ell = 0 , 1,\ldots$, are the eigenvalues of the Laplace-Beltrami operator on $\calM$. 
This approach shows that to upper bound the 1-Wasserstein distance, all we need is to know how sparse the sequence $\mu_{\ell}$, $\ell =0,1,\ldots$, is for a $G$-invariant probability measure $\mu$.  To this end, we use recent results on quantifying the sparsity of the series for $G$-invariant functions defined on a connected compact smooth manifold $(\calM,g)$ \citep{tahmasebi2023exact}. 

However, it turns out that using this method cannot guarantee a finite convergence rate since high-frequency components in the sum accumulate a lot of noise for empirical measures. To solve this issue, we use a mollifier function with exponential tail decay (in the Laplace-Beltrami basis) and use the theory of heat kernels on manifolds to achieve the final result. Further detailed explanations are provided in Appendix \ref{appendix_thrm},  where  the complete proof is presented.  

\subsection{1-Wasserstein Distance for Smooth Densities}

Assume that $\mu \in \calP(\calM)$ is absolutely continuous with respect to the uniform probability measure $dx = \frac{1}{\vol(\calM)} d\vol_g(x)$ on $(\calM,g)$. Assume that $\frac{d\mu}{dx} \in \calH^s(\calM)$, for some $s\ge 0$, where $\calH^s(\calM)$ denotes the Sobolev space of real-valued measurable functions on $(\calM,g)$ with square-integrable derivatives up to order $s$.  In this special case, the probability measure is \textit{smoother} as $s$ grows.

It turns out that in this special case, the convergence rate of estimating the 1-Wasserstein distance as a function of the number of samples can be improved using a new estimator $\tilde{\mu}$ (which is different from the modified empirical estimator $\hat{\mu}$). The following theorem states the main result for smooth distributions.

\begin{restatable}[Convergence rate of the 1-Wasserstein distance for smooth distributions under invariances]{theorem}{thrmsob}\label{thrmsob}
For any $G$-invariant probability measure $\mu \in \calP(\calM)$ with $\frac{d\mu}{dx} \in \calH^s(\calM)$ for some $s\ge 0$, there exists an estimator $\tilde{\mu} \in \calP(\calM)$,  as a function of  $n$ i.i.d. samples $X_1,X_2\ldots, X_n \sim \mu$, such that
\begin{align*}
\E[W_1(\mu,&\tilde{\mu})] \lesssim \Big ( 
\frac{\vol(\calM/G)}{n}
\Big)^{\frac{s+1}{2s+d}}
~\Big \| \frac{d\mu}{dx} \Big \|_{\calH^s(\calM)}^{\frac{d-2}{2s+d}},
\end{align*}
where $\vol(\calM/G)$ is the volume of the quotient space $\calM/G$ and $d := \dim(\calM/G)\ge 3$.  
\end{restatable}

Theorem \ref{thrmsob} shows that the gain of invariances for estimating the 1-Wasserstein distance for smooth distributions under invariances follows the same behavior as before, for any $s\ge 0$. The two-fold gain is observed in the exponent and the multiplicative factor. The new upper bound's exponent interpolates between the worst-case exponent $1/\dim(\calM)$ and $1/2$.  As $s$ grows, the exponent converges to $1/2$, as expected. Moreover, if $G = \{\text{id}_G\}$, the bound reduces to the known convergence rate of 1-Wasserstein distance estimation under smoothness (without invariances) \citep{liang2021well, niles2022minimax}.

\textbf{Proof sketch for Theorem} \ref{thrmsob}.
To define the estimator $\tilde{\mu}$, we need to review some facts about manifolds. The set of square-integrable $G$-invariant functions on a connected, compact, smooth manifold $\calM$ has an orthonormal basis $\phi_{\ell}^{\text{inv}} \in L^2(\calM)$, $\ell =0,1,\ldots$ that consists of the eigenfunctions of the Laplace-Beltrami operator on $\calM$ (see Appendix \ref{appnd_prel} for more details). For example, for the  circle $\calM = \mathbb{S}^1$, these functions correspond to the sinusoidal waves that are invariant under the group action. Given $n$ samples $X_1,X_2,\ldots,X_n$, the Borel measure $\tilde{\mu}$ is defined using its Radon-Nikodym derivative with respect to the uniform probability measure on $(\calM,g)$ as follows:
\begin{align}
\frac{d\tilde{\mu}}{dx}:= \frac{1}{n}\sum_{\ell=0}^{T-1} \sum_{i=1}^n\phi_\ell^{\text{inv}}(X_i)\phi_{\ell}^{\text{inv}},
\end{align}
where $T$ is a fixed positive integer (to be set). For any $T$, $\tilde{\mu}$ is a Borel measure, but generally, it can be a signed measure. We can then take the closest probability measure to $\tilde{\mu}$ in the 1-Wasserstein distance as the final estimation for $\mu$. With a slight abuse of notation, we denote the final output of the algorithm by $\tilde{\mu}$ again.

The parameter $T$ indicates when the sum is terminated. Note that estimating the higher-order coefficients in the Fourier basis requires many samples, and so if the Fourier coefficients decay quickly, one can neglect them and truncate the sum at a finite $T$ with a small error. Indeed,  choosing a  higher number $T$ of eigenfunctions reduces the {bias} of the estimator, but it also increases the {variance} due to the randomness of sampling and the difficulty of estimating the higher-order terms.
Therefore, optimizing $T$ to balance the bias and variance terms, according to the problem's parameters, allows us to achieve the best algorithm of this type (in terms of the convergence rate). Note that the decay of the Fourier coefficients is provided using the Sobolev space assumption (see Appendix  \ref{appendix_thrmsob} more details). We follow this approach to prove Theorem \ref{thrmsob}.

\subsection{Sobolev Integral Probability Metrics (Sobolev IPMs)}

Sobolev Integral Probability Metrics (IPMs) are a family of integral probability metrics \citep{muller1997integral, sriperumbudur2012empirical} that  use the Sobolev space $\calH^\alpha(\calM)$ as the set of test functions  to define a divergence between probability measures.

\begin{definition}[Sobolev IPMs] For any $\alpha\ge0$, the Sobolev IPM with parameter $\alpha$ is defined as
\begin{align*}
    D_{\alpha}(\mu,\nu)&:=\sup_{\substack{f \in \calH^\alpha(\calM) \\ \| f \|_{\calH^\alpha(\calM)}  \le 1}} \Big \{
    \E_{x  \sim \mu} [f(x)] - \E_{x \sim \nu}[f(x)]
    \Big\},
\end{align*}
for any (Borel) probability measures $\mu, \nu \in \calP(\calM)$. 
\end{definition}

Here, we study the sample complexity of estimating Sobolev IPMs under invariances, where we assume that we have samples from a smooth probability measure $\mu$; more precisely, $\frac{d\mu}{dx} \in \calH^s(\calM)$ for some $s\ge 0$.
Note that for $\alpha>d/2$, the Sobolev IPMs are special cases of the Maximum Mean Discrepancies (MMDs) with the Reproducing Kernel Hilbert Space (RKHS) $\calH^\alpha(\calM)$. Since we cover the convergence rates of MMDs  in the next section, we focus on the non-MMD regime where $\alpha \le d/2$.

\begin{restatable}[Convergence rate of Sobolev IPMs under invariances]{theorem}{thrmsobipm}\label{thrmsobipm}
Consider a $G$-invariant probability measure $\mu \in \calP(\calM)$ with $\frac{d\mu}{dx} \in \calH^s(\calM)$ for some $s\ge 0$. 
\begin{itemize}
    \item 
If $\alpha<d/2$, then there exists an estimator $\tilde{\mu} \in \calP(\calM)$,  as a function of  $n$ i.i.d. samples $X_1,X_2\ldots, X_n \sim \mu$, such that
\begin{align*}
\E[D_{\alpha}(\mu,\tilde{\mu})] \lesssim    \Big (
\frac{\vol(\calM/G)}{n}
\Big)^{\frac{s+\alpha}{2s+d}}
~\Big \| \frac{d\mu}{dx} \Big \|_{\calH^s(\calM)}^{\frac{d-2\alpha}{2s+d}},
\end{align*}
where the constant only depends on the manifold and $\vol(\calM/G)$ is the volume of the quotient space $\calM/G$ with $d := \dim(\calM/G)$.  
 
\item 
If $\alpha=d/2$, then there exists an estimator $\tilde{\mu} \in \calP(\calM)$ such that
\begin{align*}
\E[D_{\alpha}(\mu,&\tilde{\mu})] \lesssim  \sqrt{
\frac{ \vol(\calM/G)\log(n)}{n}
}.
\end{align*}
\end{itemize}
\end{restatable}

We note that the same two-fold gain is observed in the convergence of the Sobolev IPMs when $\alpha \le d/2$.

The convergence results behave differently in the three regimes. First, if $\alpha<d/2$, it extends the 1-Wasserstein distance convergence rate for smooth distributions to the  Sobolev IPMs. Second, if $\alpha>d/2$, as we will see in the next section, the rate saturates at $O(n^{-1/2})$. 
Finally, in the third convergence regime where $\alpha = d/2$, we get a convergence rate of the order $O(\sqrt{\log(n)/n})$.

\begin{remark}
    While Theorem \ref{thrmsobipm} holds when (at least) $\frac{d\mu}{dx} \in L^2(\calM)$, we can extend it to \textit{all} (Borel) probability measures (for which the rate corresponds to $s=0$), by  the same approach as the proof of Theorem \ref{thrm}.
\end{remark}

\textbf{Proof sketch for Theorem} \ref{thrmsobipm}. We follow the same approach as Theorem \ref{thrmsob}. Here, the coefficients of the test functions $f \in \calH^\alpha(\calM)$  in the Laplace-Beltrami basis can be upper bounded using the parameter $\alpha$ (by the definition of Sobolev spaces). Then, we use an appropriate frequency $T$, as a function of $\alpha$, to simultaneously minimize the bias and variance terms.

\subsection{Maximum Mean Discrepancy (MMD)}

The Maximum Mean Discrepancy (MMD) is an integral probability metric associated with a Positive Definite Symmetric (PDS) kernel, where the set of test functions is a Reproducing Kernel Hilbert Space (RKHS) \citep{berlinet2011reproducing}.

\begin{definition}
    For any PDS kernel $K: \calM \times \calM \to \R$ with an RKHS denoted by $\calH$, define 
    \begin{align*}
         D_{\calH}(\mu,\nu)&:= \sup_{\substack{f \in \calH \\ \| f \|_{\calH}  \le 1}} \Big \{
    \E_{x  \sim \mu} [f(x)] - \E_{x \sim \nu}[f(x)]
    \Big\}, 
\end{align*}
for any (Borel) probability measures $\mu, \nu \in \calP(\calM)$.
\end{definition}

 Note that when $\alpha > d/2$, the Sobolev IPMs are MMDs, since  in that case, $\calH^\alpha(\calM)$ is an RKHS (Sobolev Inequality). For simplicity, we assume that the  PDS kernel $K$ is diagonalizable in the Laplace-Beltrami basis:
\begin{align}
    K(x,y)=\sum_{\ell=0}^\infty \xi_\ell \phi_\ell(x) \phi_\ell(y),
\end{align}
for some $\xi_{\ell} \in \R$. This holds, for example, for the dot-product kernels on the sphere $\mathbb{S}^{d-1}$ as well as all Sobolev spaces $\calH^\alpha(\calM)$ with $\alpha > d/2$.

To study the gain of invariances for estimating MMDs, we need to define a quantity that measures how many eigenfunctions can take non-zero coefficients when restricted to the $G$-invariant functions. This weighted quantity must also depend on the kernel decomposition in the Laplace-Beltrami operator. 
To this end, define
\begin{align}
    \tr(K;G) := \sum_{\lambda\neq 0}^\infty m(\lambda;G) |\xi_\ell(\lambda)|,
\end{align}
where the sum is over all the eigenvalues of the Laplace-Beltrami operator on the manifold $(\calM,g)$, and $m(\lambda; G)$ denotes the multiplicity of the eigenvalue $\lambda$ when restricted to invariant eigenfunctions (see Appendix \ref{appnd_prel} for more details).

The following theorem related the quantity $\tr(K;G)$ to the sample complexity gain of invariances for the MMD estimation problem.  

\begin{restatable}{theorem}{thrmmmd}\label{thrmmmd} 
Consider a $G$-invariant probability measure $\mu \in \calP(\calM)$ with $\frac{d\mu}{dx} \in \calH$. Then, there exists an estimator $\tilde{\mu} \in \calP(\calM)$ such that
\begin{align}
    \mathbb{E}[D_{\calH} (\mu,\hat{\mu})] \le \sqrt{\frac{\tr(K;G)}{n}}.
\end{align}
\end{restatable}

This allows us to immediately conclude the following result for the Sobolev IPMs with $\alpha>d/2$ (i.e., the MMD regime).

\begin{corollary}[Gain of invariances for the Sobolev IPMs in the MMD regime] For any $\alpha>d/2$, there exists an estimator $\tilde{\mu} \in \calP(\calM)$ such that
\begin{align*}
\E[D_{\alpha}(\mu,&\tilde{\mu})] \le  \sqrt{ 
\frac{
\mathcal{Z}(\alpha;G)}{n} 
},
\end{align*}
where $\mathcal{Z}(\alpha;G)$ is the zeta function associated with the Lie group action $G$  on the manifold $\calM$:
\begin{align}
    \calZ(\alpha;G):= \sum_{\lambda \neq 0} m(\lambda;G)\lambda^{-\alpha}.
\end{align}

\end{corollary}

Note that if $\alpha >d/2$ (i.e., the MMD regime), the exponent is always $\frac{1}{2}$ and it does not depend on the group of invariances.
Moreover, the gain of invariances in this case can intuitively be understood as having \textit{effectively}
\begin{align}
    n \times \frac{\mathcal{Z}(\alpha)}{\mathcal{Z}(\alpha;G)}
\end{align}
 samples instead of $n$ samples, where $\mathcal{Z}(\alpha)$ is the original zeta function of the manifold (achieved via the trivial group $G = \{\text{id}_G\}$).  Therefore, the quantity $\mathcal{Z}(\alpha)/\mathcal{Z}(\alpha;G)$ represents the sample complexity gain of invariances in  the kernel regime.  

We provide an example to see how the behavior of the quantity  $\mathcal{Z}(\alpha)/\mathcal{Z}(\alpha; G)$ can depend on more delicate information about the group action (i.e., more than the dimension of the group and the volume of the quotient space).

\begin{example}
    Consider the unit circle $\mathbb{S}^1$ with the round metric. Note that if we parameterize functions defined on $\mathbb{S}^1$ as functions of the angle $\theta \in [0,2\pi)$, then the eigenfunctions of the Laplace-Beltrami operator on $\mathbb{S}^1$ are the sinusoidal functions:
    \begin{align}
        \phi_k(\theta) = \exp(ik\theta) \quad \text{with} \quad \lambda = k^2; \quad k \in \mathbb{Z}.
    \end{align}
    This allows us to compute the zeta function of the unit circle as
    \begin{align}
        \calZ(\alpha) &= \sum_{\lambda \neq 0}^\infty m(\lambda)\lambda^{-\alpha} = 2\sum_{k=1}^\infty k^{-2\alpha} = 2 \zeta(2\alpha),
    \end{align}
    where $\zeta$ denotes the Riemann zeta function. 

    Now consider the finite group of rotations around the center: $
        G: = \Big\{\frac{2\pi j}{|G|}: j \in \mathbb{Z} \Big\}.$ 
        An eigenfunction $\phi_k$ is invariant with respect to the action of $G$ if and only if we have
        \begin{align}
            \phi_k(\theta) = \phi_k\Big(\theta + \frac{2\pi j}{|G|}\Big), 
        \end{align}
        for any $\theta \in [0,2\pi)$ and $j\in \mathbb{Z}$. This is equivalent to having  $k$ divide $|G|$. Thus, we get
        \begin{align}
           \calZ(\alpha;G) &= \sum_{\lambda \neq 0}^\infty m(\lambda; G)\lambda^{-\alpha} = 2|G|^{-2\alpha}\sum_{k=1}^\infty k^{-2\alpha} \\
           &= 2|G|^{-2\alpha} \zeta(2\alpha).
        \end{align}
        This shows that the gain of invariances in sample complexity, in this case, is to have effectively
        \begin{align}
            n \times |G|^{2\alpha}
        \end{align}
        many samples instead of $n$ samples. Note that $\alpha>d/2$ in the MMD  regime.  Therefore, this example shows how the multiplicative gain can be improved from $|G|$ to $|G|^{2\alpha}$ in the MMD regime, which is in contrast to
        the problem of estimating Sobolev IPMs in the non-MMD regime (Theorem~\ref{thrmsobipm}).
\end{example}

\begin{example}\label{exm::torus}
Let $\mathbb{T}^d=[0,1]^d$ denote the  $d$-dimensional flat torus, and consider the isometric action of the group $G = \mathbb{T}^{\dim(G)}$ with the circular shift on the first $\dim(G)$ coordinates of $\mathbb{T}^d$; for any $\tau \in G$ and $x \in \mathbb{T}^d$:
\begin{align}
  \tau x :=&\big(\tau_1+x_1,\tau_2+x_2,\ldots,\tau_{\dim(G)}+x_{\dim(G)},\\
  &x_{\dim(G)+1},\ldots, x_d\big),  
\end{align}
where the coordinates are summed modulo one.
Consider the heat kernel on $\mathbb{T}^d$, defined as:
\begin{align*}
    K_\beta(x,y) = \sum_{0\neq v \in \mathbb{Z}^d} \exp(-\beta \|v\|^2_2 +2\pi i \langle 
    v,x + y\rangle).
\end{align*}
    Note that 
    \begin{align}
        \tr(K_\beta;G)&= \sum_{0\neq v \in \mathbb{Z}^{d-\dim(G)}} \exp(-\beta \|v\|^2_2)\\&= \Theta_\beta^{(d-\dim(G))}-1,
    \end{align}
    where 
    \begin{align}
        \Theta_\beta := \sum_{v \in \mathbb{Z}}\exp(-\beta v^2)>1.
    \end{align}
    This shows that the gain of having invariances with respect to the group $G$ for estimating the MMD with the heat kernel is to have  effectively
    \begin{align*}
        n \times \frac{\tr(K; \{\id_G\})}{\tr(K;G)} &= n \times \frac{\Theta_\beta^d-1}{\Theta_{\beta}^{d-\dim(G)}-1} \\&\approx n \times \Theta_{\beta}^{\dim(G)},
    \end{align*}
    many samples, instead of $n$. This shows that the gain of invariances can exponentially depend on the Lie group dimension $\dim(G)$ while estimating MMDs.
\end{example}

\textbf{Proof sketch for Theorem} \ref{thrmmmd}. The spectrum of the test functions $f \in \calH$,  in the Laplace-Beltrami basis, decays quickly in the MMD case, and it can be shown that we can take $T \to \infty$, and use a similar approach to Theorem \ref{thrmsobipm}  to prove the convergence result.

\subsection{Density Estimation in $L^2(\calM)$}

In this section, we present convergence results for estimating probability density functions $\frac{d\mu}{dx}$ in $L^2(\calM)$ distance while having $n$ i.i.d. samples from $\mu \in \calP(\calM)$. 

\begin{restatable}[Convergence rate for density estimation in $L^2(\calM)$ distance under invariances]{theorem}{thrmdenl}\label{thrmdenl}
Consider a $G$-invariant probability measure $\mu \in \calP(\calM)$ with $\frac{d\mu}{dx} \in \calH^s(\calM)$. Then, there exists an estimator $\tilde{\mu} \in \calP(\calM)$ such that
\begin{align*}
\E \Big[  \Big \| \frac{d\tilde{\mu}}{dx} - &\frac{d\mu}{dx} \Big  \|_{L^2(\calM)} \Big] \lesssim   
\Big (  
\frac{\vol(\calM/G)}{n}
\Big)^{\frac{s}{2s+d}}
~\Big \| \frac{d\mu}{dx} \Big \|_{L^2(\calM)}^{\frac{d}{2s+d}},
\end{align*}
where the constant only depends on the manifold and $\vol(\calM/G)$ denotes the volume of the quotient space $\calM/G$ with  $d := \dim(\calM/G)\ge 3$.

\end{restatable}

We observe that the density estimation problem  in the $L^2(\calM)$ distance exploits the same two-fold gain as previous cases.

\subsection{Density Estimation in $L^\infty(\calM)$}

In this part, we study the gain of group invariances for the density estimation problem in the $L^{\infty}(\calM)$ distance.

\begin{restatable}[Convergence rate for density estimation in $L^\infty(\calM)$  under invariances]{theorem}{thrmdenlinf}\label{thrmdenlinf}
Consider a $G$-invariant probability measure $\mu \in \calP(\calM)$ with $\frac{d\mu}{dx} \in \calH^s(\calM)$ for some $s>d/2$. Then, there exists an estimator $\tilde{\mu} \in \calP(\calM)$,  such that
\begin{align*}
\E \Big[  \Big \| \frac{d\tilde{\mu}}{dx} - &\frac{d\mu}{dx} \Big  \|_{L^\infty(\calM)} \Big] \\&\lesssim 
\Big ( 
\frac{(\vol(\calM/G))^{\frac{2s}{s-d/2}}}
{n}
\Big)^{\frac{s-d/2}{2s+d}}
~\Big \| \frac{d\mu}{dx} \Big \|_{\calH^s(\calM)}^{\frac{d}{s+d/2}},
\end{align*}
where the constant only depends on the manifold.

\end{restatable}

 Note that here we only get a convergence result if $s>d/2$; otherwise, the density $\frac{d\mu}{dx}$ is not necessarily bounded, since $\calH^s(\calM)$ contains unbounded functions in that case (Sobolev inequality).

\begin{remark}
       Note that in the  density estimation in $L^2(\calM)$ distance, for finite groups, the effective number of samples is $n \times |G|$, as intuitively expected (for all $s,d$). However, for the $L^\infty(\calM)$ distance, better rates are achievable according to  Theorem \ref{thrmdenlinf}. For instance, let us take $d=4$ and $s = 6$; then, for finite groups, the convergence rate becomes
    \begin{align}
        \E \Big[  \Big \| \frac{d\tilde{\mu}}{dx} - &\frac{d\mu}{dx} \Big  \|_{L^\infty(\calM)} \Big] \lesssim   \Big ( 
\frac{1}
{n |G|^3}
\Big)^{1/4}.
    \end{align}
    This shows that the effective number of samples can potentially be more than the \textit{linear gain}, i.e., in this example, we have an effective number of samples $n \times |G|^3$. 
\end{remark} 

We also notice that in this paper we used a specific simple algorithm to achieve the bound for the density estimation in  $L^\infty(\calM)$, and it  can be improved when there is no invariances  \citep{uppal2019nonparametric}. We leave the problem of proving tighter  bounds for this case under invariances to future works.

\section{Experiments}

\begin{figure}\label{fig::exp}
\centering
\includegraphics[scale=0.7]{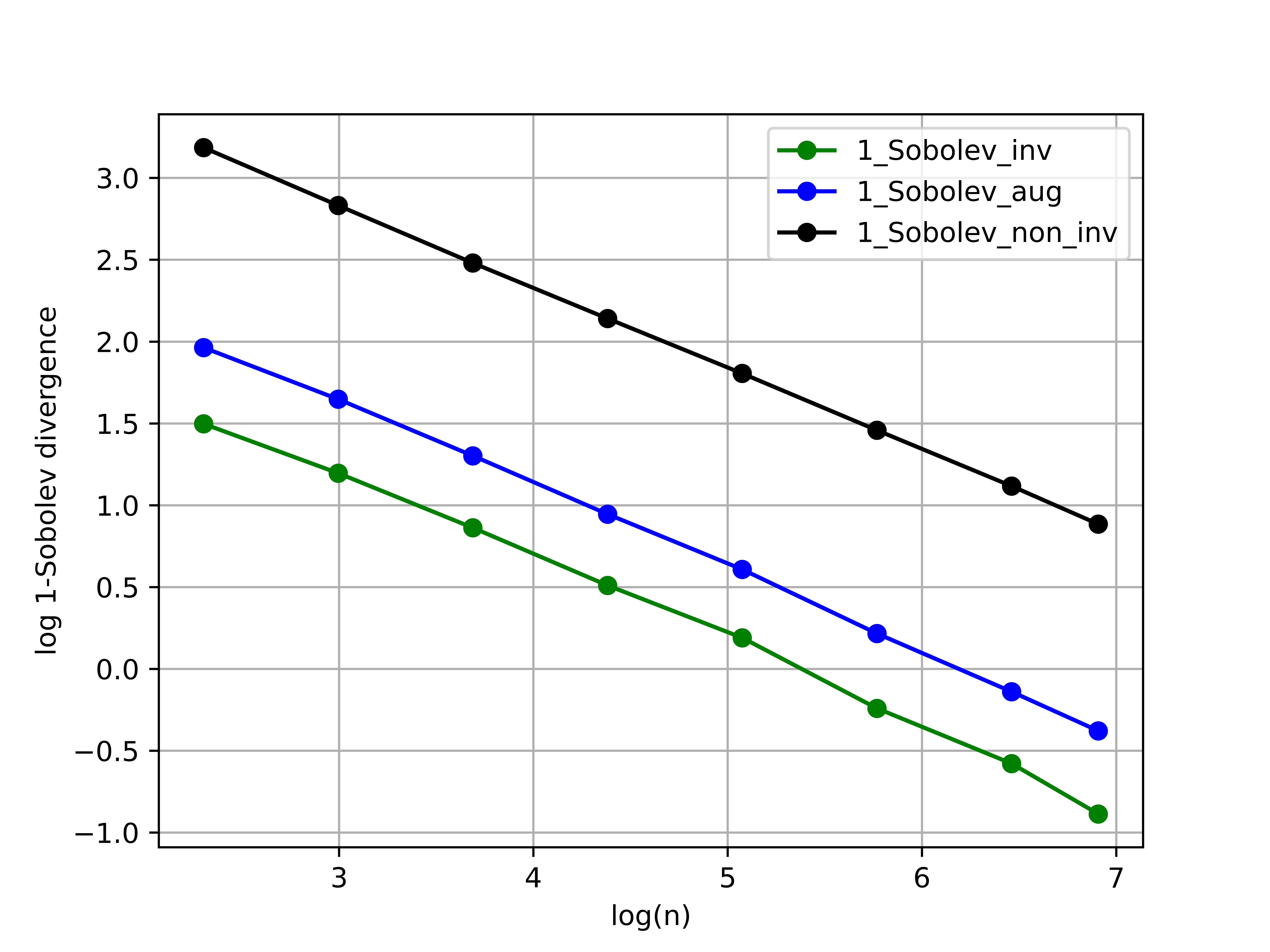}
\caption{The sample complexity gain of estimation the Sobolev IPM with $\alpha = 1$ for invariant distributions.}
\end{figure}

To observe the gain of invariances, we conduct a simple experiment on  the following synthetic dataset.   The input space is the six-dimensional flat torus $\mathbb{T}^6 = [0,1]^6$, and the group of invariances acts as the circular shifts on the last two coordinates; see \cref{exm::torus} for more details about this action. We consider a non-invariant distribution $\mu_{\text{non}}$ and an invariant distribution $\mu_{\text{inv}}$ in this setting as follows. Samples from the non-invariant distribution are generated based on the sum of three i.i.d. random variables $X = X_1 + X_2 + X_3$, each chosen uniformly from $[0,1/3]^6$. Moreover, samples from the invariant distribution such as   $(X, X_5, X_6)$ are generated based on the sum of four i.i.d. random variables, denoted by $X = X_1+X_2+X_3+X_4$, each chosen uniformly from $[0, 0.25]^4$, and for the last two coordinates, $X_5, X_6$ are chosen uniformly from $[0,1]$. With this specific choice, one can show that the distributions of $\mu_{\text{non}}$
 and $\mu_{\text{inv}}$ both lie in the Sobolev space with $s = 6$.
 
 The convergence rates for estimating the Sobolev IPM with $\alpha = 1$ are shown in Figure \ref{fig::exp}, where we used the proposed estimators in the paper to achieve the results. Indeed, we didn't optimize the parameter $T$, and just used a fixed regularization parameter, which leads to having parallel plots in the logarithmic scale. Still, the gain of invariances is evident. We also report the convergence rates for the full data augmentation setting in the plot. Note that even though the group of invariances here has infinitely many elements, we can report the full data augmentation results using the Fourier approach. The results show that full data augmentation is not enough to achieve the performance of the invariant estimator that is used in the paper. This shows the crucial role of regularization (i.e., balancing the bias and variance terms) in the problem.

 \section{Conclusion}

 In this paper, we studied the sample complexity gain of group invariances for estimating various divergence measures on manifolds, given a number of i.i.d. samples. We extended the recent theoretical results for finite group actions on submanifolds (of full dimension) of $\R^d$ to arbitrary manifolds (including spheres, tori, hyperbolic spaces, etc.), and arbitrary groups (including groups of positive dimension).  The gain of invariances in sample complexity is two-fold: improving the exponent and multiplying the number of samples by a factor depending on the group of invariances.
We also studied faster convergence rates for smooth distributions and showed a similar gain in that  case.

\section*{Acknowledgments}
 The authors extend their appreciation to Philippe Rigollet for his valuable inspiration. 
 This research was supported by the Office of Naval Research award N00014-20-1-2023 (MURI ML-SCOPE), NSF award CCF-2112665 (TILOS AI Institute), NSF award 2134108, and the Alexander von Humboldt Foundation.

\bibliography{OT}
\bibliographystyle{plainnat}

\newpage
\appendix
\onecolumn


\section{Preliminaries for Proofs}\label{appnd_prel}

This section reviews some essential preliminaries on group actions on manifolds \citep{tahmasebi2023exact}. We consider a compact connected boundaryless smooth Riemannian manifold $(\calM, g)$ with dimension $\dim(\calM)$. The Riemannian metric $g$ allows us to define the volume element $d\vol_g(x)$ on $\calM$, and without loss of generality, we assume that it is normalized such that $\vol(\calM)= \int_{\calM} d \vol_g(x) = 1$. The volume element is also denoted by $dx$ whenever the choice of the Riemannian metric is evident from the context. 

The Laplace-Beltrami operator $\Delta_g$ is the unique continuous operator satisfying the integration by parts formula:
\begin{align}
    \int_{\calM} \langle \nabla_g \phi(x), \nabla_g \psi(x) \rangle_{L^2(\calM)} d\vol\nolimits_g(x) =  \int_{\calM}  \Delta_g\phi(x) \psi(x) d\vol\nolimits_g(x),
\end{align}
for any smooth function $\phi,\psi:\calM \to \R$. One can see that the Euclidean Laplacian $\Delta = \sum_{i} \partial_i^2$ satisfies this definition by Green's identities.  

The Laplace-Beltrami operator eigenfunctions, denoted by $\phi_\ell$, $\ell=0,1,\ldots$, are the the sequence of functions satisfying the equation $\Delta_g \phi + \lambda_{\ell} \phi = 0$ on the manifold $\calM$, where $0=\lambda_0<\lambda_1\le \lambda_2\le \ldots$ denote the eigenvalues of $(-\Delta_g)$. Let us assume that $\|\phi_\ell\|_{L^2(\calM)} = \int_\calM \phi^2_\ell(x) d\vol_g(x) = 1$, i.e., the eigenfunctions are normalized in $L^2(\calM)$. One important property of the Laplace-Beltrami  eigenfunctions it that they constitute a basis for the function space $L^2(\calM)$.

It can be shown that the sequence $\lambda_ \ell \to \infty$ as $\ell \to \infty$, thus it does not accumulate around any finite number.  Moreover, according to the celebrated Weyl's law, the density of eigenvalues is given by the following asymptotic formula:
\begin{align}
    N(\lambda):= \#\{ \lambda_\ell \le \lambda\} = \frac{\omega_{\dim(\calM)}}{(2\pi)^{\dim(\calM)}}\vol(\calM)\lambda^{\dim(\calM)/2}+ O(\lambda^{(\dim(\calM)-1)/2}),
\end{align}
where $\omega_{\dim(\calM)}$ is the volume of the unit $\dim(\calM)$-dimensional ball in the Euclidean space $\R^{\dim(\calM)}$ (with the usual metric). Even more, one can show that Weyl's law holds locally around each point on the manifold:
\begin{align}
    N_x(\lambda):= \sum_{\lambda_\ell \le \lambda} \phi^2_{\ell}(x) = \frac{\omega_{\dim(\calM)}}{(2\pi)^{\dim(\calM)}}\vol(\calM)\lambda^{\dim(\calM)/2}+ O(\lambda^{(\dim(\calM)-1)/2}),
\end{align}
where the error term is uniformly bounded for all $x \in \calM$. Note that integrating the above formula on the manifolds gives the original version of Weyl's law.

A Lie group $G$ acts on $\calM$ isometrically if for all $\tau\in G$, the corresponding map $x \mapsto \tau x$ on the manifold is an isometry.
A function $f:\calM \to \calM$ is called $G-$invariant, if and only if one has $f(x) = f(\tau x)$, for all $\tau \in G$, in $L^2(\calM)$. 

For any function $f \in L^2(\calM)$, one has 
$f = \sum_{\ell=0}^\infty \langle f, \phi_\ell \rangle_{L^2(\calM)} \phi_\ell$. However, for any $G$-invariant function, one can show that many coefficients in the above series are a priory known to be zero. Indeed, for any eigenspace of the Laplace-Beltrami operator, such as $V_{\lambda}$, with eigenvalue $\lambda$, there exists a subspace $V_{\lambda,G}\subseteq V_\lambda$ such that all $G$-invariant functions $f \in V_\lambda$ can be written as a linear combination of functions in $V_{\lambda,G}$ \citep{tahmasebi2023exact}. Without loss of generality, we assume that the eigenfunctions $\phi_\ell$ are chosen such that for any $\lambda$ there exists a subset of $\{ \phi_\ell: \ell = 0, 1,\ldots\}$ that forms an orthonormal basis for $V_{\lambda,G}$. Let $\mathbbm{1}_G(\ell)$ indicate whether $\phi_\ell \in V_{\lambda,G}$ (with one) or not (with zero).  
Let $N_x(\lambda;G):= \sum_{\ell: \lambda_\ell \le \lambda}  \mathbbm{1}_G(\ell)\phi_{\ell}^2(x)$. 

We are interested in asymptotic estimation formulae for $N_x(\lambda;G)$ to study the gain of invariances. Note that if $G=\{\text{id}_G\}$, then Weyl's law provides such estimation. The following theorem shows how to extend it to non-trivial groups.

\begin{theorem}[Local Weyl's law for $G$-invariant functions, \citep{tahmasebi2023exact}]\label{weyl_inv} For any Lie group $G$ acting smoothly on a connected compact boundaryless manifold $\calM$, one has for all $x \in \calM$,
    \begin{align}
        N_x(\lambda;G) = \sum_{\ell: \lambda_\ell \le \lambda} \mathbbm{1}_G(\ell)\phi^2_{\ell}(x) = \frac{\omega_{d}}{(2\pi)^{d}}\vol(\calM/G)\lambda^{d/2}+ O(\lambda^{(d-1)/2}),
    \end{align}
    where $d:=\dim(\calM/G)$ and $\vol(\calM/G)$ denote the dimension and the volume of the principal part of the quotient space, respectively. Note that one always has $N_x(\lambda;G)\le N_x(\lambda)$.  
\end{theorem}

A Hilbert space $\calH \subseteq L^2(\calM)$ is called a Reproducing Kernel Hilbert Spaces (RKHS) if and only if for any $x\in \calM$ and $f \in \calH$, one has $|f(x)| \le C_x\|f\|_{\calH}$ for some constant $C_x <\infty$. Associated with $\calH$, there exists a PDS kernel $K:\calM \times \calM \to \R$ , and according to  Mercer's theorem, it can be diagonalized in some basis in $L^2(\calM)$. In this paper, for the sake of simplicity, we consider kernels that can be diagonalized in the Laplace-Beltrami basis as follows:
\begin{align}
    K(x,y) = \sum_{\ell=0}^\infty \xi_\ell \phi_\ell(x) \phi_\ell(y),
\end{align}
for some coefficients $\xi_\ell \in \R$. This includes dot-product kernels over spheres. We further assume that the coefficients, for each $\ell$, only depend on $\lambda_\ell$, thus we can denote them as $\xi(\lambda)$, too. A set of sufficient conditions for  kernels on manifolds to satisfy such types of diagonalization is given in \citep{tahmasebi2023exact}. 

We also use the following definition of Sobolev spaces $\calH^s(\calM)$ for any $s\ge0$:
\begin{align}
    \calH^s(\calM):= \Big \{ f \in L^2(\calM): \|f\|_{\calH^s(\calM)}^2:= \langle f , \phi_0\rangle_{L^2(\calM)}^2 +\sum_{\ell=1}^\infty \lambda_\ell^s \langle f , \phi_\ell\rangle_{L^2(\calM)}^2 < \infty \Big \}.
\end{align}
Sobolev spaces are indeed Hilbert spaces under the following inner product:
\begin{align}
    \langle f_1 , f_2 \rangle_{\calH^s(\calM)}^2:= \langle f_1 , \phi_0\rangle_{L^2(\calM)} \langle f_2 , \phi_0\rangle_{L^2(\calM)}
    +
    \sum_{\ell=1}^\infty \lambda_\ell^s \langle f_1 , \phi_\ell\rangle_{L^2(\calM)} \langle f_2 , \phi_\ell\rangle_{L^2(\calM)}. 
\end{align}
Note that Sobolev spaces are Reproducing Kernel Hilbert Spaces if and only if $s > d/2$ \citep{tahmasebi2023exact}. For any $0\le \alpha \le \beta <\infty$, one has $\calH^\beta(\calM)\subseteq \calH^\alpha(\calM) \subseteq L^2(\calM)$. Moreover, $H^0(\calM)= L^2(\calM)$.

\section{Proof of Theorem \ref{thrm}}\label{appendix_thrm}

\thrm*

\begin{proof} Let $\phi_\ell$, $\ell=0,1,\ldots$, denote the eigenfunctions of the Laplace-Beltrami operator,  corresponding to the eigenvalues $0=\lambda_0 < \lambda_1\le \lambda_2 \le \ldots$. Each eigenspace is denoted by $V_{\lambda_{\ell}}$, and its subspace of $G$-invariant eigenfunctions is denoted by $V_{\lambda_{\ell},G}$ (see Appendix \ref{appnd_prel} for more details). 
For any probability measure $\mu \in \calP(\calM)$, define 
$\mu_{\ell} = \int_{\calM} \phi_{\ell} d\mu$ for each $\ell = 0,1,\ldots$, and note that $\mu_0=1$. 
Let $dx = \frac{1}{\vol(\calM)}d\vol_g(x)$ denote the uniform measure on the manifold and define the inner-product in $L^2(\calM)$ with respect to the measure $d\vol_g(x)$. If a probability measure $\mu$ is absolutely continuous with respect to the uniform measure, then $\frac{d\mu}{dx}$ denotes its Radon-Nikodym derivative. 
\begin{lemma}\label{lemma_sob_1}
Consider two Borel measures $\mu,\nu \in \calP(\calM)$  with $\int_{\calM} d\mu = \int_{\calM} d\nu =1$. 
Also, assume that $\frac{d\mu}{dx} , \frac{d\nu}{dx}\in L^2(\calM)$ with $\vol(\calM)=1$. 
 Then, 
\begin{align}
W_1(\mu,\nu) \le  D_1(\mu,\nu),
\end{align}
where $D_1(\mu,\nu)$ is the Sobolev IPM with parameter $\alpha=1$.  
 \end{lemma}
\begin{proof}
Note that  $\Lip(\calM) \subset L^2(\calM)$. By Rademacher's theorem \citep{evans2015measure}, any $f \in \Lip(\calM)$ is differentiable almost everywhere, and  thus $| \nabla_g f(x) |_g \le 1$ for almost every $x \in \calM$. Therefore, 
\begin{align}
\| \nabla_g f \|_{L^2(\calM)}^2 & \le \int_{\calM} | \nabla_g f(x) |_g^2 d\text{vol}_g(x) \le \int_{\calM}  d\text{vol}_g(x) = 1.
\end{align} 
For any $f \in L^2(\calM)$, one can write $f = \sum_{\ell=0}^\infty f_{\ell}\phi_{\ell}$  with $f_{\ell}:= \langle f, \phi_{\ell}\rangle_{L^2(\calM)}$. Thus, we have
\begin{align}
W_1(\mu,\nu) &= \sup_{f \in \Lip(\calM) } \Big \{ \int_{\calM} f d\mu  -  \int_{\calM} f d\nu \Big \}  \\&\le \sup_{\substack{f \in L^2(\calM) \\ \| \nabla_g f \|_{L^2(\calM)}^2 \le 1}} \Big \{ \int_{\calM} f d\mu  -  \int_{\calM} f d\nu \Big \}.
\end{align}
Now note that since $\mu_0=\nu_0$, one can restrict  the above optimization problem to functions with $f_0=0$. For those functions, we have $\| f\|_{\calH^1(\calM)} = \| \nabla_g f\|_{L^2(\calM)} = 1$, and thus we have
\begin{align}
W_1(\mu,\nu)  \le \sup_{\substack{f \in \calH^1(\calM) \\ \| f \|_{\calH^1(\calM)\le 1}}}\Big \{ \int_{\calM} f d\mu  -  \int_{\calM} f d\nu \Big \} = D_1(\mu,\nu).
\end{align}
\end{proof}
To prove Theorem \ref{thrm}, we use Lemma \ref{lemma_sob_1}. First, note that it is impossible to immediately apply Lemma \ref{lemma_sob_1} since   the unknown measure $\mu$ is not necessarily absolutely continuous with respect to the uniform measure on the manifold. Nevertheless, let us define the sequence  
\begin{align}
    \tilde{\mu}_\ell := \frac{1}{n}\sum_{i=1}^n \mathbbm{1}_G(\ell)\phi_{\ell}(X_i),
\end{align}
for each $\ell=0,1,\ldots$.


Given the probability measure $\mu \in \calP(\calM)$, define the twisted measure $\mu_\star$ as follows:
\begin{align}
   \frac{d\mu_\star}{dx} = \sum_{\ell=0}^\infty \exp(-\sigma\lambda_\ell )\mu_\ell \phi_\ell,
\end{align}
where $\sigma$ is a fixed positive constant (to be set later). Note that for a fixed $\sigma$, the new measure $\mu_\star$ is absolutely continuous with respect to the uniform measure. Indeed,  its density is even smooth (i.e., having smooth derivatives with arbitrary orders), since the tail of the above summation goes to zero exponentially fast as $\ell \to \infty$. Also, note that $\int_\calM d\mu_\star = 1$, and with a bit more consideration, we can observe that $\mu_\star$ is a probability measure. Let us also define $\tilde{\mu}_\star$ similarly, where $\mu_\ell$ is replaced with its empirical estimation as above. Thus, the measure $\tilde{\mu}_\star$ is computable from the given samples.

There is also another interpretation for the probability measure $\mu_\star$. Let $X\sim \mu$ denote a sample from $\mu$, and consider the heat diffusion corresponding to the Brownian motion on the manifold, started from $X$ and stopped at time $t=\sigma$. The Brownian motion is assumed to be independent of the sample $X$, and it can be shown that the law of the process at time $t = \sigma$ is the probability measure $\mu_\star$. Therefore, this gives a natural coupling between the two measures $\mu$ and $\mu_\star$. 

Hence, we can write 
\begin{align}
    W_1(\mu,\tilde{\mu}_\star) &\le W_1(\mu,\mu_\star)  + W_1(\mu_\star,\tilde{\mu}_\star) 
    \\
    & \le C\sqrt{\sigma} +W_1(\mu_\star,\tilde{\mu}_\star), 
\end{align}
where we use the fact that for any probability measure $\nu \in \calP(\calM)$, one has $W_1(\nu,\nu_\star)\le W_2(\nu,\nu_\star)\le C\sqrt{\sigma}$ for some constant $C$ depending on the manifold. The last inequality is due to the fact that the law of the Brownian motion on the manifold at time $t = \sigma$ behaves (locally) similarly to the Gaussian distribution on Euclidean spaces.

Let us now bound $\E[W_1(\mu_\star,\tilde{\mu}_\star)]$. Since the two measures $\mu_\star,\tilde{\mu}_\star$ are absolutely continuous with respect to the uniform measure on the manifold, we can use Lemma \ref{lemma_sob_1} and Lemma \ref{lemma_sob_2}; thus,
\begin{align}
    \E[W_1(\mu_\star,\tilde{\mu}_\star)] \le  \sqrt{\E \Big [\sum_{\ell=1}^{\infty} \exp(-2\sigma\lambda_\ell )\frac{(\mu_{\ell} - \tilde{\mu}_{\ell})^2}{\lambda_{\ell}}
    \Big ]}. \label{wass_eq}
\end{align}
Let us define \begin{align}
    R(\lambda):=\E \Big[ \sum_{\ell: \lambda_\ell \le  \lambda  } (\mu_{\ell} - {\tilde{\mu}}_{\ell})^2 \Big],
\end{align}
and note that we can rewrite Equation (\ref{wass_eq}) as follows:
\begin{align}
    \E[W_1(\mu_\star,\tilde{\mu}_\star)] \le  \sqrt{
\int_{0^+}^\infty \exp(-2\sigma\lambda)\frac{dR(\lambda)}{\lambda}
},
\end{align}
where the above formula must be understood as a Riemann-Stieltjes integral.

Using integration by parts and  Lemma \ref{equ_sob_ipm_err_1}, we have 
\begin{align}
    \int_{0^+}^\infty \exp(-2\sigma\lambda)\frac{dR(\lambda)}{\lambda} = \underbrace{ \exp(-2\sigma\lambda) \frac{R(\lambda)}{\lambda}\Big |_{0^+}^\infty}_{=0} + \int_{0^+}^\infty R(\lambda)  \frac{2\sigma \lambda+1}{\lambda^2}\exp(-2\sigma\lambda) d\lambda.
\end{align}
Now we use Lemma \ref{equ_sob_ipm_err_1} and Theorem \ref{weyl_inv} to find an upper bound on the above integral:
\begin{align}
    \int_{0^+}^\infty R(\lambda)  \frac{2\sigma \lambda+1}{\lambda^2} \exp(-2\sigma\lambda)d\lambda &\lesssim \frac{1}{n}\int_{0^+}^\infty \frac{\omega_{d}}{(2\pi)^{d}}\vol(\calM/G)\lambda^{d/2} \frac{2\sigma \lambda+1}{\lambda^2} \exp(-2\sigma\lambda)d\lambda \\ &= \frac{1}{n}\frac{\omega_{d}}{(2\pi)^{d}}\vol(\calM/G) \int_{0^+}^\infty (2\sigma\lambda^{d/2-1} + \lambda^{d/2-2})\exp(-2\sigma\lambda)d\lambda.
\end{align}
By a change of variables in the above integral, we conclude that
\begin{align}
    \int_{0^+}^\infty R(\lambda)  \frac{2\sigma \lambda+1}{\lambda^2} \exp(-2\sigma\lambda) d\lambda &\lesssim \frac{\vol(\calM/G) }{n}\sigma^{1-d/2}. 
\end{align}
Therefore, 
\begin{align}
    \E[W_1(\mu,\tilde{\mu}_\star)] \lesssim \sqrt{\sigma} + \sqrt{\frac{\vol(\calM/G) }{n}}\sigma^{(2-d)/4}. 
\end{align}
To minimize the above upper bound, we consider the function $p(\sigma) = \sqrt{\sigma} + \sqrt{\frac{\vol(\calM/G) }{n}}\sigma^{(2-d)/4}$ for $\sigma \in (0,\infty)$ that achieves its minimum at 
\begin{align}
    \sigma = \sigma_{\text{opt}}  = \Big ((1-d/2)^2 \frac{\vol(\calM/G)}{n}\Big)^{2/d},
\end{align}
and thus we have 
\begin{align}
    \E[W_1(\mu,\tilde{\mu}_\star)] \lesssim \Big(\frac{\vol(\calM/G) }{n}\Big )^{1/d},
\end{align}
which completes the proof.
\end{proof}

\section{Proof of Theorem \ref{thrmsob}}\label{appendix_thrmsob}

\thrmsob*

\begin{proof}
According to Lemma \ref{lemma_sob_1},  we have the upper bound
\begin{align}
W_1(\mu,\nu) \le  D_1(\mu,\nu),
\end{align}
for any measures $\frac{d\mu}{dx} , \frac{d\nu}{dx}\in L^2(\calM)$. Now consider the estimator $\mu^\star$  for $\mu$, for which we are given $n$ 
i.i.d samples,  defined in the proof of Theorem \ref{thrmsobipm}. From the assumption, we have $\frac{d\mu}{dx} \in \calH^s(\calM)$ and clearly, $\frac{d\tilde{\mu}}{dx}\in \calH^s(\calM)$ (from its own definition). Thus, we can use the proof of Theorem \ref{thrmsobipm} specified to $\alpha =1$ to achieve the result for the 1-Wasserstein distance of smooth distributions.


\end{proof}

\section{Proof of Theorem \ref{thrmsobipm}}

\thrmsobipm*

\begin{proof} 
We first need to prove the following upper bound on the Sobolev IPMs.

\begin{lemma}\label{lemma_sob_2}
For Borel measures $\mu,\nu$ with $\int_{\calM} d\mu = \int_{\calM} d\mu =1$, such that $\frac{d\mu}{dx} , \frac{d\nu}{dx}\in L^2(\calM)$ and assuming $\vol(\calM)=1$, one has 
\begin{align}
D_{\alpha}^2(\mu,\nu) =  \sum_{\ell=1}^{\infty} \frac{(\mu_{\ell} - \nu_{\ell})^2 }{\lambda^{\alpha}_{\ell}},
\end{align}
for any $\alpha \ge0$, where $\mu_\ell= \E_{x\sim \mu}[\phi_\ell(x)]$ and $\nu_\ell= \E_{x\sim \nu}[\phi_\ell(x)]$ for any $\ell$.
\end{lemma}
\begin{proof}
Any function $f \in \calH^{\alpha}(\calM)$ can be written as $f = \sum_{\ell=0}^\infty f_{\ell}\phi_{\ell}$  with $\|f \|_{\calH^{\alpha}(\calM)}^2=f_0^2+\sum_{\ell=1}^\infty \lambda_{\ell}^\alpha f_{\ell}^2 < \infty$. From the definition,
\begin{align}
D_{\alpha}(\mu,\nu) &= \sup_{\substack{f \in \calH^\alpha(\calM) \\ \| f \|_{\calH^\alpha(\calM)\le 1}} } \Big \{ \int_{\calM} f d\mu  -  \int_{\calM} f d\nu \Big \}  \\ & =  \sup_{\substack{f \in \calH^\alpha(\calM) \\ \| f \|_{\calH^\alpha(\calM)\le 1}}} \Big \{ \int_{\calM} f \frac{d\mu}{dx} dx  -  \int_{\calM} f \frac{d\nu}{dx} dx \Big \} 
\\& =   \sup_{\substack{f \in \calH^\alpha(\calM) \\ \| f \|_{\calH^\alpha(\calM)\le 1}}} \Big \langle f, \frac{d\mu}{dx} - \frac{d\nu}{dx} \Big \rangle_{L^2(\calM)} \\
& = \sup_{\substack{f \in \calH^\alpha(\calM) \\ \| f \|_{\calH^\alpha(\calM)\le 1}}} 
\sum_{\ell=0}^{\infty} f_{\ell}(\mu_{\ell} - \nu_{\ell})
\\
& \overset{(a)}{=} \sup_{\substack{f \in \calH^\alpha(\calM) \\ \| f \|_{\calH^\alpha(\calM)\le 1}}} 
\sum_{\ell=1}^{\infty} \lambda_\ell^{\alpha/2}f_{\ell}\times \frac{(\mu_{\ell} - \nu_{\ell})}{\lambda_\ell^{\alpha/2}}
\\
& \overset{(b)}{=} \sup_{\substack{f \in \calH^\alpha(\calM) \\ \| f \|_{\calH^\alpha(\calM)\le 1}}} 
\sqrt{ \underbrace{\sum_{\ell=1}^{\infty} \lambda_\ell^{\alpha}f_{\ell}^2}_{\le 1}}
\times
\sqrt{\sum_{\ell=1}^{\infty}
\frac{(\mu_{\ell} - \nu_{\ell})^2}{\lambda_\ell^{\alpha}}}
\\
& =
\sqrt{\sum_{\ell=1}^{\infty}
\frac{(\mu_{\ell} - \nu_{\ell})^2}{\lambda_\ell^{\alpha}}},
\end{align}
where (a) is due to the fact that $\mu_{0} = \nu_0=1$, and  (b) follows from the Cauchy–Schwarz inequality.
\end{proof}

\begin{proof}[Proof of Theorem \ref{thrmsobipm}]
Given i.i.d. samples from a probability measure $\mu$ with $\| \frac{d\mu}{dx}\|_{\calH^s(\calM)} < \infty$, we propose the following estimator. First, fix a parameter $T$ (to be set later) and let $\mathbbm{1}_G(\ell;T) = \mathbbm{1}_G(\ell)\mathbbm{1}\{ \ell < T\}$ for any $\ell=0,1,\ldots$.  Then, define the sequence  
\begin{align}
    \tilde{\mu}_\ell := \frac{1}{n}\sum_{i=1}^n \mathbbm{1}_G(\ell;T)\phi_{\ell}(X_i),
\end{align}
for each $\ell$ and let $\frac{d\tilde{\mu}}{dx}:= \sum_{\ell=0}^\infty \tilde{\mu}_\ell \phi_\ell$. The measure $\tilde{\mu}$ is thus well-defined, and $\int_\calM d\tilde{\mu}=1$, but it is not necessarily a probability measure since it can be a signed measure. However, due to the truncation of the sum at $T$, we  have $\frac{d\tilde{\mu}}{dx} \in \calH^s(\calM)$. Define a new probability measure $\mu^\star$ as 
\begin{align}
     \mu^\star = \argmin_{\nu \in \calP(\calM)} D_{\alpha}(\nu, \tilde{\mu}).
\end{align}
This shows that by the triangle inequality
\begin{align}
    D_\alpha(\mu,\mu^\star) \le D_\alpha(\mu,\tilde{\mu}) + D_{\alpha}(\tilde{\mu}, \mu^\star) \le 2D_\alpha(\mu,\tilde{\mu}),
\end{align}
where in above, we used the definition of the probability measure $\mu^\star$.

We claim that the estimated probability measure achieves the convergence rate claimed in the theorem. To this end,  we need to prove upper bounds on $\E[D_{\alpha}(\mu,\tilde{\mu})]$. According to Lemma \ref{lemma_sob_2}, one has
\begin{align}
    \E[D_\alpha(\mu,\mu^\star)] &\le 2 \E[D_\alpha(\mu,\tilde{\mu})] \\
    & \le  2 \sqrt{\E[D^2_\alpha(\mu,\tilde{\mu})]}  \\
    & = 2 \sqrt{\E \Big[\sum_{\ell=1}^{\infty} \frac{(\mu_{\ell} - {\tilde{\mu}}_{\ell})^2 }{\lambda^{\alpha}_{\ell}} \Big]}.
\end{align}
Note that
\begin{align}
    \E \Big[ \sum_{\ell=1}^{\infty} \frac{(\mu_{\ell} - {\tilde{\mu}}_{\ell})^2 }{\lambda^{\alpha}_{\ell}} \Big] &= \E \Big[ \sum_{\ell=1}^{T-1} \frac{(\mu_{\ell} - {\tilde{\mu}}_{\ell})^2 }{\lambda^{\alpha}_{\ell}} \Big] + \E \Big[ \sum_{\ell \ge T} \frac{(\mu_{\ell} - {\tilde{\mu}}_{\ell})^2 }{\lambda^{\alpha}_{\ell}} \Big] \\
    & =  \underbrace{
    \E \Big[ \sum_{\ell=1}^{T-1} \frac{(\mu_{\ell} - {\tilde{\mu}}_{\ell})^2 }{\lambda^{\alpha}_{\ell}} \Big]
    }_{\text{(I)}}+ \underbrace{
     \sum_{\ell \ge T} \frac{\mu_{\ell}^2 }{\lambda^{\alpha}_{\ell}} 
    }_{\text{(II)}},\label{equ_sob_ipm_err}
\end{align}
since $\tilde{\mu}_\ell=0$ for $\ell >T$. 
\end{proof}
 To upper bound the second term, we write \begin{align}
         \sum_{\ell \ge T} \frac{\mu_{\ell}^2 }{\lambda^{\alpha}_{\ell}} &= \sum_{\ell>T} \mu_{\ell}^2 \lambda^{s}_{\ell}\lambda^{-(s+\alpha)}_{\ell}  \le \lambda_T^{-(s+\alpha)}   \sum_{\ell \ge T} \mu_{\ell}^2 \lambda^{s}_{\ell} \le \lambda_T^{-(s+\alpha)} \Big \| \frac{d{\mu}}{dx} \Big \|^2_{\calH^s(\calM)}.
    \end{align}
To bound the first term, we need to use the following lemma. 

\begin{lemma}\label{equ_sob_ipm_err_1} One has 
\begin{align} 
    R(\lambda):=\E \Big[ \sum_{\ell: \lambda_\ell \le  \lambda} (\mu_{\ell} - {\tilde{\mu}}_{\ell})^2 \Big]\le \frac{1}{n} 
\E_{x \sim \mu} \Big[N_x(\lambda;G) \Big], 
\end{align}
    where $N_x(\lambda;G)$ is defined in Theorem \ref{weyl_inv}.
\end{lemma}

\begin{proof}
Note that the coefficients $\tilde{\mu}_\ell$ are empirical estimates of $\mu_\ell$ given $n$ i.i.d. samples; thus 
    \begin{align}
        R(\lambda_T)=\E \Big[ \sum_{\ell: \lambda_\ell \le  \lambda_{T}} (\mu_{\ell} - {\tilde{\mu}}_{\ell})^2 \Big] &= \frac{1}{n} \Big \{
\sum_{\ell: \lambda_\ell \le \lambda_T} 
\mathbbm{1}_G(\ell)
\Big(\E_{x \sim \mu}[\phi^2_\ell(x)] - (\E_{x \sim \mu}[\phi_\ell(x)])^2 \Big)
        \Big \}\\
        & \le \frac{1}{n} 
\E_{x \sim \mu} \Big[\sum_{\ell: \lambda_\ell \le \lambda_T} \mathbbm{1}_G(\ell) \phi^2_\ell(x) \Big]  \\ & \le \frac{1}{n} 
\E_{x \sim \mu} \Big[N_x(\lambda_T;G) \Big].  
    \end{align}
\end{proof}
We use Lemma \ref{equ_sob_ipm_err_1} to complete the proof.  First, note that
\begin{align}
   \E \Big[ \sum_{\ell=1}^{T-1} \frac{(\mu_{\ell} - {\tilde{\mu}}_{\ell})^2 }{\lambda^{\alpha}_{\ell}} \Big] \le \int_{0^+}^{\lambda_T} \frac{dR(\lambda)}{\lambda^\alpha},
\end{align}
where $R(\lambda)$ is defined in Lemma \ref{equ_sob_ipm_err_1} and $dR(\lambda)$ denotes the Stieltjes measure corresponding to $R(\lambda)$. By integration by parts, one has
\begin{align}
    \E \Big[ \sum_{\ell=1}^{T-1} \frac{(\mu_{\ell} - {\tilde{\mu}}_{\ell})^2 }{\lambda^{\alpha}_{\ell}} \Big] \le \int_{0^+}^{\lambda_T} \frac{dR(\lambda)}{\lambda^\alpha} = \frac{R(\lambda)}{\lambda^\alpha} \Big |_{0^+}^{\lambda_T} + \alpha \int_{0^+}^{\lambda_T} \frac{R(\lambda)}{\lambda^{\alpha+1}}d\lambda.
\end{align}
Now assume $\alpha<d/2$. By Theorem \ref{weyl_inv} and Lemma  \ref{equ_sob_ipm_err_1}, we have
\begin{align}
    \E \Big[ \sum_{\ell=1}^{T-1} \frac{(\mu_{\ell} - {\tilde{\mu}}_{\ell})^2 }{\lambda^{\alpha}_{\ell}} \Big] &\le  \frac{R(\lambda_T)}{\lambda_T^\alpha}  + \alpha \int_{0^+}^{\lambda_T} \frac{R(\lambda)}{\lambda^{\alpha+1}}d\lambda\\
    & \le \frac{1}{n} \Big \{ \frac{\omega_{d}}{(2\pi)^{d}}\vol(\calM/G)\lambda_T^{d/2-\alpha}+ O(\lambda_T^{(d-1)/2-\alpha}) \\
    &\quad + \alpha \int_{0^+}^{\lambda_T} \Big \{\frac{\omega_{d}}{(2\pi)^{d}}\vol(\calM/G)\lambda_T^{d/2-\alpha-1}+ O(\lambda_T^{(d-1)/2-\alpha-1})\Big \}d\lambda
    \Big \} \\
    & \lesssim \frac{1}{n}\Big(1 + \frac{\alpha}{d/2-\alpha}\Big)\frac{\omega_{d}}{(2\pi)^{d}}\vol(\calM/G)\lambda_T^{d/2-\alpha},
\end{align}
where in $\lesssim$ the constant is absolute. Plugging in this estimation into Equation (\ref{equ_sob_ipm_err}) results in
\begin{align}
    \E \Big[ \sum_{\ell=1}^{\infty} \frac{(\mu_{\ell} - {\tilde{\mu}}_{\ell})^2 }{\lambda^{\alpha}_{\ell}} \Big] & \lesssim \frac{1}{n} \frac{d/2}{d/2-\alpha}\frac{\omega_{d}}{(2\pi)^{d}}\vol(\calM/G)\lambda_T^{d/2-\alpha} + \lambda_T^{-(s+\alpha)} \Big \| \frac{d{\mu}}{dx} \Big \|^2_{\calH^s(\calM)}.
\end{align}
We can choose the parameter $\lambda_T \in (0,\infty)$ to minimize the above upper bound. 

Note that the function $p(\lambda) = c_a\lambda^{-a} + c_b\lambda^b$ with $a,b,c_a,c_b>0$ is minimized for $\lambda \in (0,\infty)$ when 
\begin{align}
    \lambda = \lambda_{\text{opt}}:=\Big ( \frac{a c_a}{b c_b} \Big )^{1/(a+b)}.
\end{align}
Taking $a = s+\alpha$, $b= d/2-\alpha$, and
\begin{align}
    c_a = \Big \| \frac{d{\mu}}{dx} \Big \|^2_{\calH^s(\calM)}, \quad c_b = \frac{1}{n} \frac{d/2}{d/2-\alpha}\frac{\omega_{d}}{(2\pi)^{d}}\vol(\calM/G), 
\end{align} 
suggests to take
\begin{align}
    \lambda_T = \Big \{
\Big((s+\alpha) \| d{\mu}/dx  \|^2_{\calH^s(\calM)}\Big)\Big/\Big(\frac{d}{2n}\frac{\omega_{d}}{(2\pi)^{d}}\vol(\calM/G)
    \Big )\Big \}^{1/(s+d/2)}.  
\end{align}
Therefore, we have the following upper bound
\begin{align}
 \E \Big[ &\sum_{\ell=1}^{\infty} \frac{(\mu_{\ell} - {\tilde{\mu}}_{\ell})^2 }{\lambda^{\alpha}_{\ell}} \Big]  \\& \lesssim  \frac{1}{n} \frac{d/2}{d/2-\alpha}\frac{\omega_{d}}{(2\pi)^{d}}\vol(\calM/G)\Big \{
\Big((s+\alpha) \| d{\mu}/dx  \|^2_{\calH^s(\calM)}\Big)\Big/\Big(\frac{d}{2n}\frac{\omega_{d}}{(2\pi)^{d}}\vol(\calM/G)
    \Big )\Big \}^{(d/2-\alpha)/(s+d/2)} \\
 &+  \Big \| \frac{d{\mu}}{dx} \Big \|^2_{\calH^s(\calM)} \Big \{
\Big((s+\alpha) \| d{\mu}/dx  \|^2_{\calH^\alpha(\calM)}\Big)\Big/\Big(\frac{d}{2n}\frac{\omega_{d}}{(2\pi)^{d}}\vol(\calM/G)
    \Big )\Big \}^{-(s+\alpha)/(s+d/2)},  
\end{align}
which shows that
\begin{align}
    \E \Big[ \sum_{\ell=1}^{\infty} \frac{(\mu_{\ell} - {\tilde{\mu}}_{\ell})^2 }{\lambda^{\alpha}_{\ell}} \Big]   \lesssim   \Big (
\frac{\vol(\calM/G)}{n}
\Big)^{\frac{s+\alpha}{s+d/2}}
~\Big \| \frac{d\mu}{dx} \Big \|_{\calH^s(\calM)}^{\frac{d-2\alpha}{s+d/2}}.
\end{align}
This completes the proof for $\alpha<d/2$, since 
\begin{align}
    \E[D_\alpha(\mu,\mu^\star)] \le 2 \sqrt{\E \Big[\sum_{\ell=1}^{\infty} \frac{(\mu_{\ell} - {\tilde{\mu}}_{\ell})^2 }{\lambda^{\alpha}_{\ell}} \Big]} \lesssim \Big (
\frac{\vol(\calM/G)}{n}
\Big)^{\frac{s+\alpha}{2s+d}}
~\Big \| \frac{d\mu}{dx} \Big \|_{\calH^s(\calM)}^{\frac{d-2\alpha}{2s+d}}.
\end{align}
Now, we study the case where $\alpha=d/2$. By Theorem \ref{weyl_inv} and Lemma  \ref{equ_sob_ipm_err_1}, we have
\begin{align}
    \E \Big[ \sum_{\ell=1}^{T-1} \frac{(\mu_{\ell} - {\tilde{\mu}}_{\ell})^2 }{\lambda^{\alpha}_{\ell}} \Big] &\le  \frac{R(\lambda_T)}{\lambda_T^\alpha}  + \alpha \int_{0^+}^{\lambda_T} \frac{R(\lambda)}{\lambda^{\alpha+1}}d\lambda \\
    & \le   \frac{\alpha}{n}  \int_{\lambda_1}^{\lambda_T} \frac{\omega_{d}}{(2\pi)^{d}}\vol(\calM/G)\frac{d\lambda}{\lambda} +O(1/n)\\
    & = \frac{\alpha}{n} \frac{\omega_{d}}{(2\pi)^{d}}\vol(\calM/G)\ln(\lambda_T) +O(1/n).
\end{align}
Thus, from Equation (\ref{equ_sob_ipm_err}) we have
\begin{align}
    \E \Big[ \sum_{\ell=1}^{\infty} \frac{(\mu_{\ell} - {\tilde{\mu}}_{\ell})^2 }{\lambda^{\alpha}_{\ell}} \Big] & \lesssim \frac{\alpha}{n} \frac{\omega_{d}}{(2\pi)^{d}}\vol(\calM/G)\ln(\lambda_T) + \lambda_T^{-(s+\alpha)} \Big \| \frac{d{\mu}}{dx} \Big \|^2_{\calH^s(\calM)} +O(1/n).
\end{align}
To minimize the above upper bound as a function of $\lambda_T$, we consider the function $p(\lambda) = c_a\lambda^{-a} + c_b \ln(\lambda)$ and assume that $c_b$ is small enough. Note that the minimum of $p(\lambda)$ for $\lambda \in (0,\infty)$ is achieved when 
\begin{align}
    \lambda = \lambda_{\text{opt}}:= \Big ( \frac{a c_a}{c_b} \Big )^{1/a}.
\end{align}
This means that for $a = s+\alpha$, $c_a =  \| d{\mu}/dx  \|^2_{\calH^s(\calM)}$, and $c_b = \frac{\alpha}{n} \frac{\omega_{d}}{(2\pi)^{d}}\vol(\calM/G)$, we can take 
\begin{align}
    \lambda_T = \Big \{
\Big((s+\alpha) \| d{\mu}/dx  \|^2_{\calH^s(\calM)}\Big)\Big/\Big(\frac{\alpha}{n}\frac{\omega_{d}}{(2\pi)^{d}}\vol(\calM/G)
    \Big )\Big \}^{1/(s+d/2)},  
\end{align}
to achieve
\begin{align}
    \E \Big[ \sum_{\ell=1}^{\infty} \frac{(\mu_{\ell} - {\tilde{\mu}}_{\ell})^2 }{\lambda^{\alpha}_{\ell}} \Big] & \lesssim \frac{d\log(n)}{(2s+d)n} \frac{\omega_{d}}{(2\pi)^{d}}\vol(\calM/G) + O(1/n).
\end{align}
This completes the proof for $\alpha=d/2$, since 
\begin{align}
    \E[D_\alpha(\mu,\mu^\star)] \le 2 \sqrt{\E \Big[\sum_{\ell=1}^{\infty} \frac{(\mu_{\ell} - {\tilde{\mu}}_{\ell})^2 }{\lambda^{\alpha}_{\ell}} \Big]}   \lesssim  \sqrt{
\frac{ \vol(\calM/G)\log(n)}{n}
}.
\end{align} 
\end{proof}

\section{Proof of Theorem \ref{thrmmmd}}

\thrmmmd*

\begin{proof}
    Note that a function $f \in \calH$ can be written as $f= \sum_{\ell=0}^\infty f_\ell \phi_\ell$ such that $\|f\|^2_\calH = \sum_{\ell=0}^\infty \frac{f_\ell^2}{\xi_\ell} <\infty$. Thus, for any probability measures $\mu,\nu$ such that $\frac{d\mu}{dx}, \frac{d\nu}{dx} \in \calH$, we have
    \begin{align}
        D_{\calH}(\mu,\nu)&:= \sup_{\substack{f \in \calH \\ \| f \|_{\calH}  \le 1}} \Big \{
    \E_{x  \sim \mu} [f(x)] - \E_{x \sim \nu}[f(x)]
    \Big\}\\
    & =\sup_{\substack{f \in \calH \\ \| f \|_{\calH}  \le 1}} \Big \{
    \sum_{\ell=1}^\infty  f_\ell(\mu_\ell-\nu_\ell)
    \Big\} = \sqrt{\sum_{\ell=1}^\infty  \xi_\ell(\mu_\ell-\nu_\ell)^2}.
    \end{align}
Now consider the estimator $\tilde{\mu}$ that is defined similar to the proof of Theorem \ref{thrm}; we have 
\begin{align}
    \tilde{\mu}_\ell := \frac{1}{n}\sum_{i=1}^n \mathbbm{1}_G(\ell)\phi_{\ell}(X_i),
\end{align}
for each $\ell=0,1,\ldots$, and  $\tilde{\mu}$ denotes the  measure\footnote{Indeed, one can also achieve the convergence rate  by truncating the sum over $\lambda_\ell <T$ with appropriate $T$ to make the estimation computable. } for which $\E_{x \sim \tilde{\mu}}[\phi_\ell(x)] = \tilde{\mu}_\ell$. Assume that $\xi_\ell$ depends only on $\lambda_\ell$, and thus denote it by $\xi(\lambda)$. Using Lemma \ref{equ_sob_ipm_err_1} and integration by parts, we have 
\begin{align}
    \E[D_{\calH}(\mu,\tilde{\mu})] &\le \sqrt{\E[D_{\calH}(\mu,\tilde{\mu})^2]} = \sqrt{\E \Big[\sum_{\ell=1}^\infty  \xi_\ell(\mu_\ell-\tilde{\mu}_\ell)^2\Big]}\\
    & = \sqrt{
    \int_{0^+}^\infty \xi(\lambda) dR(\lambda)
    }\\
    & \le \sqrt{\frac{1}{n}
    \int_{0^+}^\infty \xi(\lambda) dN_x(\lambda;G)
    }, 
\end{align}
which holds when $\xi(\lambda)$ is non-increasing. Now, note that
\begin{align}
    \int_{0^+}^\infty \xi(\lambda) dN_x(\lambda;G) = \sum_{\ell=1}^\infty \mathbbm{1}_G(\ell)\xi(\lambda_\ell),
\end{align}
and this completes the proof.

\end{proof}

\section{Proof of Theorem \ref{thrmdenl}}

\thrmdenl*

\begin{proof}
    For any two probability measures $\mu,\nu$ such that $\frac{d\mu}{dx}, \frac{d\nu}{dx} \in L^2(\calM)$, we have
\begin{align}
    D_0(\mu,\nu)&= \sup_{\substack{ f \in L^2(\calM) \\ \| h\|_{f^2(\calM)}=1}} \E_{\mu}[f] - \E_{\nu}[f] \\
    & = \sup_{\substack{ f \in L^2(\calM) \\ \| f \|_{L^2(\calM)}=1}} \langle  f, \frac{d\mu}{dx} - \frac{d\nu}{dx}\rangle_{L^2(\calM)}\\
    & = \Big \| \frac{d\mu}{dx} - \frac{d\nu}{dx} \Big\|_{L^2(\calM)}.
\end{align}
Therefore, we can use the estimator in the proof of Theorem \ref{thrmsobipm} for $\alpha = 0$ to achieve the desired result.
\end{proof}

\section{Proof of Theorem \ref{thrmdenlinf}}

\thrmdenlinf*

\begin{proof}
Consider the  estimator $\tilde{\mu}$ proposed in the proof of Theorem \ref{thrmsobipm} with a parameter $T$ (to be set). Let $\mu_T$ denote the measure obtained by truncating the sum corresponding to $\mu$ at $T$. Note that  for the estimator $\tilde{\mu}$ we have 
\begin{align}
    \E \Big[  \Big \| \frac{d\tilde{\mu}}{dx} -  \frac{d\mu}{dx} \Big  \|_{L^\infty(\calM)} \Big] \le \underbrace{
    \E \Big[  \Big \| \frac{d\tilde{\mu}}{dx} -  \frac{d\mu_T}{dx} \Big  \|_{L^\infty(\calM)} \Big]
    }_{\text{(I)}}+ 
    \underbrace{
     \Big \| \frac{d\mu_T}{dx} -  \frac{d\mu}{dx} \Big  \|_{L^\infty(\calM)} 
    }_{\text{(II)}}. 
\end{align}
To bound the first term, we use the Cauchy–Schwarz inequality for low degree functions; see the following lemma.
\begin{lemma}
    Assume that $f \in L^2(\calM)$ and $f = \sum_{\ell=0}^T f_\ell \mathbbm{1}_G(\ell)\phi_\ell$. Then,
    \begin{align}
        \| f\|_{L^\infty(\calM)}  \le \|f\|_{L^2(\calM)} \sqrt{\sup_{x\in \calM}N_x(\lambda;G)}.
    \end{align}
\end{lemma}
\begin{proof}
By Theorem \ref{weyl_inv}, we have
    \begin{align}
        \| f\|_{L^\infty(\calM)} &= \sup_{x \in \calM} |f(x)|= \sup_{x \in \calM} \Big|\sum_{\ell=0}^T f_\ell \mathbbm{1}_G(\ell) \phi_\ell \Big|\\
        & \le \sqrt{ \sum_{\ell=0}^T f_\ell^2}\sqrt{ \sup_{x \in \calM}\sum_{\ell=0}^T \mathbbm{1}_G(\ell) \phi_\ell(x)^2 }\\
        & \le \|f\|_{L^2(\calM)} \sqrt{\sup_{x\in \calM}N_x(\lambda;G)}.
    \end{align}
\end{proof}
Now we use the above lemma and Lemma \ref{equ_sob_ipm_err_1} to write
\begin{align}
    \E \Big[  \Big \| \frac{d\tilde{\mu}}{dx} -  \frac{d\mu_T}{dx} \Big  \|_{L^\infty(\calM)} \Big] &\le \sqrt{\sup_{x\in \calM}N_x(\lambda_T;G)} \E\Big[\Big \| \frac{d\tilde{\mu}}{dx} -  \frac{d\mu_T}{dx} \Big  \|_{L^2(\calM)}\Big]\\
    & \le \sqrt{\frac{1}{n}\sup_{x\in \calM}N_x(\lambda_T;G) \times \E_{x \sim \mu}[N_x(\lambda_T;G)]}.
\end{align}
To bound the second term, we have
\begin{align}
    \Big \| \frac{d\mu_T}{dx} -  \frac{d\mu}{dx} \Big  \|_{L^\infty(\calM)} & = \sup_{x \in \calM} \Big |\sum_{\ell>T} \mu_\ell \mathbbm{1}_G(\ell) \phi_\ell \Big |\\
    & =  \sup_{x \in \calM} \Big |\sum_{\ell>T} \mu_\ell \mathbbm{1}_G(\ell)   \lambda_\ell^{s/2 }\phi_\ell \lambda_\ell^{-s/2} \Big |\\
    & \le \sqrt{\sum_{\ell>T} \mu^2_\ell  \lambda_\ell^{s}}\sqrt{ \sup_{x \in \calM} \sum_{\ell>T} \mathbbm{1}_G(\ell)\phi_\ell^2 \lambda_\ell^{-s}}\\
    & \le   \|f\|_{\calH^s(\calM)}\sqrt{ \sup_{x \in \calM} \sum_{\ell>T} \mathbbm{1}_G(\ell)\phi_\ell^2 \lambda_\ell^{-s}}.
\end{align}
Note that
\begin{align}
    \sup_{x \in \calM} \sum_{\ell>T} \mathbbm{1}_G(\ell)\phi_\ell^2 \lambda_\ell^{-s} &= \sup_{x \in \calM}\int_{\lambda_T}^\infty \lambda_\ell^{-s}dN_x(\lambda;G)\\
    & \le \frac{d}{d-2s} \frac{\omega_d}{(2\pi)^d}\vol(\calM/G)\lambda_T^{d/2-s}.
\end{align}
Thus, we combine the two terms and obtain
\begin{align*}
\E \Big[  \Big \| \frac{d\tilde{\mu}}{dx}& -  \frac{d\mu}{dx} \Big  \|_{L^\infty(\calM)} \Big]  \\
&\le \sqrt{\frac{1}{n}\sup_{x\in \calM}N_x(\lambda;G) \times \E_{x \sim \mu}[N_x(\lambda;G)]}   +\|f\|_{\calH^s(\calM)}  \sqrt{\frac{d}{d-2s} \frac{\omega_d}{(2\pi)^d}\vol(\calM/G)\lambda_T^{d/2-s}}   \\
& \lesssim \frac{1}{\sqrt{n}}\vol(\calM/G)\lambda_T^{d/2} + \|f\|_{\calH^s(\calM)}\sqrt{\vol(\calM/G)}\lambda_T^{d/4-s/2}.
\end{align*}
By minimizing the above upper bound as a function of $\lambda_T$, we obtain the desired result.


\end{proof}

\end{document}